\documentclass{article}
\usepackage{authblk}
\usepackage{graphicx}
\usepackage{amssymb,amsmath, amsthm,mathtools,amsfonts}

\usepackage[marginratio=1:1,height=600pt,width=460pt,tmargin=100pt]{geometry}

\usepackage{hyperref}       
\usepackage{url}            
\usepackage{booktabs}       
\usepackage{amsfonts}       
\usepackage{nicefrac}       
\usepackage{microtype}      
\usepackage{subcaption}

\usepackage{epstopdf}
\usepackage{amsthm}
\usepackage{mathtools}

\DeclareMathOperator*{\argmin}{arg\,min}
\DeclareMathOperator{\Div}{div}

\DeclareMathOperator{\Tr}{Tr}

\newcommand{\E}{\mathbb{E}}

\newcommand{\CF}{\mathcal{F}}
\newcommand{\CM}{\mathcal{M}}
\newcommand{\CN}{\mathcal{N}}

\newcommand{\CP}{\mathcal{P}}
\newcommand{\CJ}{\mathcal{J}}
\newcommand{\W}{\mathcal{W}}
\newcommand{\R}{\mathbb{R}}

\usepackage{amsfonts}
\usepackage{amssymb}
\usepackage[dvipsnames]{xcolor}
\usepackage{bbm}
\usepackage{graphicx}
\usepackage{enumerate}
\usepackage{enumitem}
\usepackage{multirow}
\usepackage[capitalise]{cleveref}
\usepackage{float}

\newcommand{\norm}[1]{\left\lVert#1\right\rVert}
\newcommand{\abs}[1]{\left |#1\right |}
\usepackage{mathtools}
\newcommand\inner[2]{\langle #1, #2 \rangle}

\usepackage{subcaption}
\usepackage{algorithmic, algorithm} 
\newtheorem{theorem}{Theorem}
\newtheorem{corollary}{Corollary}
\newtheorem{assumption}{Assumption}
\newtheorem{definition}{Definition}
\newtheorem{lemma}{Lemma}
\newtheorem{proposition}{Proposition}

\newtheorem{remark}{Remark}

\usepackage[export]{adjustbox}
\usepackage{commath}
\usepackage{hyperref}
\usepackage{centernot}
\usepackage{xcolor}
\usepackage{wrapfig}
\allowdisplaybreaks

\newcommand*\diff{\mathop{}\!\mathrm{d}}
\hypersetup{
	colorlinks,
	linkcolor=black,
	anchorcolor=black,
	citecolor=black
}

\graphicspath{ {./figures/} }

\title{Robustness and Structure Preservation in Flow-Based Generative Models via Wasserstein Path-Space Divergences}

\begin{document}

\author[1]{Ziyu Chen\thanks{Email: ziyuchen@unc.edu}}
\author[2]{Markos A. Katsoulakis\thanks{Email: markos@umass.edu}}
\author[3]{Benjamin J. Zhang\thanks{Email: bjz@unc.edu}}
\affil[1]{Department of Mathematics and School of Data and Information Sciences, University of North Carolina at Chapel Hill}
\affil[2]{Department of Mathematics and Statistics, University of Massachusetts Amherst}
\affil[3]{Department of Mathematics, Rutgers University–New Brunswick}

\date{}

\maketitle

\begin{abstract}
    We introduce a novel Wasserstein-1 ($\mathcal{W}_1$) path-space divergence for stochastic and deterministic dynamics and establish a Wasserstein Uncertainty Propagation (WUP) theorem that bounds the $\mathcal{W}_1$ distance between terminal distributions by the proposed divergence, equivalently characterized by a weighted $L^2$ discrepancy between the underlying drifts and the $\mathcal{W}_1$ distance between their initial measures. A key ingredient is a probabilistic framework combining adjoint Feynman–Kac representations with synchronous coupling (and reflection coupling on bounded domains), yielding Wasserstein stability estimates beyond existing PDE- and Girsanov-based approaches. The framework accommodates time-varying and possibly degenerate diffusion coefficients, empirical and singular measures, and remains valid in the deterministic limit of flow matching. Unlike KL-based uncertainty quantification bounds, it does not require absolute continuity of path measures and therefore remains well-defined in singular settings. As consequences of the WUP theorem, we derive $\mathcal{W}_1$ robustness and generalization bounds for score-based generative models and flow matching at both population and finite-sample levels. We further specialize the framework to group-symmetric targets, providing the first error analysis of equivariant flow-based models and the first quantitative comparison between data augmentation and equivariant inductive bias. Our analysis identifies a symmetry-aware Wasserstein path-space divergence that quantifies the model-form error induced by non-equivariant parametrizations. We prove that this error cannot be removed by additional data or training and vanishes only under equivariant architectures, establishing a precise theoretical advantage of equivariant inductive bias over data augmentation. Numerical experiments on group-symmetric Gaussian mixtures corroborate the theory.
\end{abstract}

\section{Introduction}
Flow-based generative models, including score-based diffusion models (SGMs) \cite{song2019generative,ho2020denoising,song2020score,songdenoising} and flow matching \cite{lipman2022flow}, are now the dominant paradigm for high-dimensional generative modeling. By learning a drift or velocity field that transports a simple reference distribution to a complex target, these models achieve state-of-the-art performance across image, audio, molecular, and scientific data generation. Their training reduces to a tractable regression problem, score matching for SGMs and velocity matching for flow matching, which is one of the principal reasons for their practical success. Despite this empirical progress, a rigorous understanding of when and how well they generalize, and how their generalization error depends on the discrepancies between the learned drift or velocity field and its target, remains incomplete. This question becomes especially important in settings that lie beyond the scope of many existing analyses, including singular target distributions, low-dimensional data manifolds, degenerate or vanishing diffusion coefficients, finite-sample training, and targets possessing additional structure such as symmetries or geometric constraints.

A key observation motivating this work is that flow-based generative models are fundamentally dynamical objects. The generated distribution at terminal time is the endpoint of an entire stochastic or deterministic trajectory, and therefore robustness should naturally be studied at the level of path measures rather than solely through terminal-time marginals. This perspective has a long history in uncertainty quantification and sensitivity analysis, where path-space relative entropy and relative entropy rate provide powerful tools for comparing stochastic processes and controlling errors in quantities of interest \cite{dupuis2016path,harmandaris2016path,katsoulakis2017scalable,zou2026goal}. However, these information-theoretic divergences rely on absolute continuity of path measures and may become infinite or undefined when the underlying processes are singular. Such situations arise naturally in modern generative modeling, including deterministic flow matching, degenerate diffusions, empirical target distributions, and dynamics constrained to lower-dimensional manifolds.

To address these limitations, we introduce a Wasserstein path-space divergence that provides a geometric measure of discrepancy between trajectory laws and remains meaningful even when the corresponding path measures are singular. Building on this notion, we establish a Wasserstein Uncertainty Propagation (WUP) theorem that propagates trajectory-level discrepancies to terminal-time distributional errors. Specifically, the theorem bounds the Wasserstein-1 ($\W_1$) distance between generated and target distributions in terms of the proposed path-space divergence that includes discrepancies in the underlying dynamics. The resulting framework unifies deterministic and stochastic generative models, accommodates unbounded domains and degenerate noise, and naturally incorporates structured targets.

Our approach differs fundamentally from existing theoretical analyses of flow-based generative models. Most current results proceed through Kullback–Leibler (KL) divergence and Girsanov-type arguments, deriving bounds in total variation, $\chi^2$, or Wasserstein distance as downstream consequences \cite{chensampling,lee2022convergence,chen2023improved,conforti2023score,oko2023diffusion}. While powerful, these methods require absolute continuity with respect to a reference process and typically assume nondegenerate diffusion, excluding precisely the singular and deterministic regimes that motivate flow matching and many structure-preserving generative models. A separate line of work studies convergence directly in Wasserstein distance, but currently relies on restrictive assumptions such as compact domains, prescribed noise schedules, or particular discretization schemes \cite{de2022convergence,mimikos2024score}. In contrast, our path-space Wasserstein framework yields robustness and generalization guarantees under substantially weaker assumptions and applies uniformly across both diffusion-based and deterministic flow-based settings.

This paper develops a unified $\W_1$ robustness and generalization theory. The central technical contribution is the introduction of a Wasserstein path-space divergence together with the Wasserstein Uncertainty Propagation theorem on $\R^d$. As applications, we derive $\W_1$ generalization bounds for score-based diffusion models and flow matching, establish robustness guarantees that also apply for singular and structured targets, and develop a quantitative theory of symmetry-informed generative modeling. These results provide a common framework for analyzing robustness, generalization, and structure preservation across a broad class of modern generative models. We summarize our main contributions below.

\paragraph{1. Wasserstein path-space divergence and Uncertainty Propagation (\Cref{sec:Wdivergence} and \Cref{sec:wup}).}
We introduce a Wasserstein path-space UQ bound and establish a corresponding Wasserstein Uncertainty Propagation (WUP) theorem for probability flows on $\R^d$. The starting point is a $\W_1$ stability estimate for two probability flows $\rho_1, \rho_2$ driven by different (differentiable) drift terms $\mathbf{b}_1,\mathbf{b}_2$ and starting from different initial measures $m_1, m_2$, evolving under a common but possibly time-varying and degenerate diffusion. The bound has the form
\begin{equation}\label{eq:UQ_bound}
    \W_1(\rho_1(T), \rho_2(T)) \le e^{cT}\!\left(\W_1(m_1, m_2) + \sqrt{T}\, \norm{\mathbf{b}_1 - \mathbf{b}_2}_{L^2(\rho_2)} \right),
\end{equation}
with a strengthened version when $b_1$ is the gradient of a strongly convex potential. The right-hand side naturally induces a Wasserstein path-space divergence between probability flows, and \eqref{eq:UQ_bound} shows that this trajectory-level discrepancy propagates to a quantitative bound on the terminal-time distributions. The proof is fully probabilistic: we represent the test functions appearing in the Kantorovich--Rubinstein dual representation of the $\W_1$ via the Feynman--Kac formula and control their Lipschitz constants through pathwise coupling estimates. This approach offers four advantages over the PDE-based methodology of \cite{mimikos2024score}:
\begin{enumerate}
    \item it applies to unbounded $\mathbb{R}^d$, not just on the torus or other compact domains;
    \item it accommodates time-varying diffusion coefficients $\sigma(t)$;
    \item it remains valid when $\sigma(t) \equiv 0$, covering the noiseless setting of flow matching, where the Feynman--Kac formula degenerates to the method of characteristics; and
    \item it imposes no regularity on the initial measures $m_1$ and $m_2$, which may be empirical, singular, or supported on a low-dimensional manifold.
\end{enumerate}
On bounded domains, a reflection coupling refinement (\Cref{appendix:torus}) yields a Lipschitz bound that decays exponentially in time, sharpening the corresponding result in \cite{mimikos2024score}. Moreover, the quantity appearing on the right-hand side of \eqref{eq:UQ_bound} satisfies the defining properties of a divergence on path space (\Cref{prop:pathspace}). The estimate therefore provides a direct mechanism for propagating path-space uncertainty to terminal-time distributional error. Unlike classical path-space UQ bounds based on relative entropy and the Csiszár–Kullback–Pinsker inequality \cite{dupuis2016path,zou2026goal}, the resulting Wasserstein divergence remains finite without any absolute continuity assumptions, making it applicable to singular measures such as empirical distributions and measures supported on low-dimensional manifolds.

\paragraph{2. $\W_1$ generalization bounds for SGMs and flow matching (\Cref{thm:generalization_error} and \Cref{thm:generalization_error_FM}).}
Applying the WUP theorem to the standard Ornstein--Uhlenbeck SGM yields a $\W_1$ bound between the target distribution and the generated distribution, decomposed into an early-stopping contribution, a term involving the initial mismatch with the Gaussian prior, the empirical denoising score matching loss, and the empirical $\W_1$ sample complexity rate (see \Cref{thm:generalization_error}). The same framework applied to the deterministic flow-matching dynamics yields an analogous $\W_1$ bound in a setting where Girsanov-based arguments do not apply (see \Cref{thm:generalization_error_FM}). While direct Wasserstein bounds for flow matching have been established before \cite{bentonerror,fukumizu2025flow,zhou2025error,kunkel2026distribution}, our bound is obtained as a direct corollary of the WUP, without requiring bounded support, adding a small amount of noise, or a particular construction of the velocity field. Together, \Cref{thm:generalization_error} and \Cref{thm:generalization_error_FM} show that the WUP provides a single, unified vehicle for the generalization analysis of stochastic and deterministic flow-based generative models.

\paragraph{3. Equivariant generative modeling (\Cref{sec:equivariant_SGM}).}
When the target distribution is invariant under the action of a group $G$, our framework specializes to give the first rigorous error analysis for SGMs and flow matching with symmetry, and the first quantitative comparison between data augmentation and equivariant score approximations. Adopting a model-form uncertainty quantification perspective, we decompose the $\W_1$ generalization error into four interpretable sources: (i) the deviation from equivariance (DFE) of the learned vector field, (ii) the score matching error against the symmetrized target, (iii) the $\W_1$ sample complexity under group symmetry \cite{chen2023sample,tahmasebisample,chen2025robust}, and (iv) an early-stopping and finite-horizon contribution. The DFE is, to our knowledge, the first quantitative measure of the model-form error introduced by a non-equivariant score parametrization, and an additive decomposition of the score matching objective (\Cref{prop:equidecomposition}) shows that it sits on top of the irreducible score matching error and cannot be reduced by augmentation. Combining this with structural results on equivariant score matching and the improved $\W_1$ sample complexity rates yields a strict inequality: equivariant parametrization captures the same statistical efficiency gain as data augmentation while additionally eliminating the model-form error from non-equivariance. Experiments on a $G$-invariant mixture of four Gaussians in $\mathbb{R}^2$ confirm the theory.

\subsection{Related work}
\label{sec:related-work}
\paragraph{Path-space UQ bounds for stochastic dynamics.}
A substantial body of work has developed uncertainty quantification and sensitivity analysis directly at the level of path measures. In particular, \cite{dupuis2016path,katsoulakis2017scalable,zou2026goal} introduce goal-oriented uncertainty quantification bounds based on relative entropy or the KL divergence, establishing quantitative control of observable discrepancies between stochastic processes and scalable sensitivity estimates for high-dimensional systems. A central theme of this literature is that divergences defined on trajectory space provide a natural mechanism for propagating model perturbations to errors in quantities of interest. Our work adopts a similar path-space viewpoint but replaces the information-theoretic framework with a geometric one. We introduce a Wasserstein path-space divergence that remains meaningful when path measures are mutually singular, a situation that commonly arises in deterministic flow models and generative modeling with singular targets such as empirical observations and those with low-dimensional supports. Moreover, our Wasserstein UQ bound also applies to deterministic dynamics without noise when the KL bound can fail.

\paragraph{Convergence and generalization of SGMs and flow matching.}
The quality of a generated distribution is typically measured by a probability divergence or distance. For SGMs, a first line of work bounds the KL divergence between the target and the generated distribution via Girsanov's theorem and then derives weaker bounds in TV, $\chi^2$, or $\W_1$ as downstream consequences \cite{chensampling,lee2022convergence,chen2023improved,conforti2023score,oko2023diffusion}. These bounds, however, require the target to be absolutely continuous with respect to a reference measure, which excludes singular targets supported on low-dimensional manifolds and the noiseless flow matching regime. A second line works directly in $\W_1$: \cite{de2022convergence} establishes convergence under the manifold hypothesis but relies on a specific discretization, and \cite{mimikos2024score} introduces an uncertainty quantification perspective using Bernstein-type PDE estimates that are restricted to compact domains and require the presence of a Laplacian term. For flow matching models, \cite{bentonerror,zhou2025error} provide $\W_2$ error bounds when the target has a bounded support, and \cite{fukumizu2025flow} establishes $W_p$ bounds for $1\leq p\leq2$ on $[-1,1]^d$ under smoothness assumptions. Closer to our setting, \cite{kunkel2026distribution} derives a bound similar to our \Cref{thm:generalization_error_FM}, but under the additional assumptions that the target is sufficiently smooth and that its covariance has controlled decay; these assumptions exclude singular measures and ensure the velocity field is uniformly Lipschitz, making their result a special case of our framework. Proposition 3 in \cite{albergobuilding} establishes a similar bound for flow matching to our \Cref{thm:generalization_error_FM}, albeit under the Wasserstein-2 metric. However, their argument does not appear to extend directly to a Wasserstein-2 bound for SGMs. Our $\W_1$ bounds for SGMs build on the UQ perspective of \cite{mimikos2024score} but remove the compact-domain and noise-nondegeneracy assumptions, and additionally cover deterministic flow matching models. The stochastic-interpolant framework of \cite{albergo2025stochastic} provides a related Girsanov-type KL discrepancy estimate using properties of Fokker-Planck equations, which require a nondegenerate diffusion coefficient; by contrast, our WUP theorem accommodates both the noiseless setting and singular initial measures.

\paragraph{Symmetry-preserving generative models.}
A growing body of empirical work has demonstrated the benefits of incorporating group symmetry directly into generative models, including structure-preserving GANs \cite{birrell2022structure}, equivariant normalizing flows \cite{kohler2020equivariant,garcia2021n}, equivariant flow matching \cite{klein2024equivariant}, and equivariant diffusion models \cite{lu2024diffusion}. Theoretical guarantees, however, remain limited. Existing theory has primarily focused on group-invariant GANs \cite{chen2023statistical}. For SGMs and flow matching, \cite{lu2024diffusion} and \cite{klein2024equivariant} show that the generated distribution is $G$-invariant whenever the score parametrization or velocity field is $G$-equivariant. However, these works do not quantify the resulting gains in statistical efficiency or compare equivariant parametrization against data augmentation. Meanwhile, \cite{hoogeboom2022equivariant} provides empirical evidence on molecular generation tasks showing that diffusion models with equivariant architectures outperform those relying solely on data augmentation. Our work bridges this gap by providing the first rigorous error analysis for symmetry-aware SGMs and flow matching, and by analytically demonstrating that equivariant parametrization strictly outperforms data augmentation.

\subsection{Organization}\label{sec:organization}

\Cref{sec:background} briefly introduces the background and setup. We propose a novel Wasserstein path-space UQ bound in \Cref{sec:Wdivergence} and show that it defines a divergence in the path-space. \Cref{sec:wup} establishes the Wasserstein Uncertainty Propagation (WUP) theorem on $\mathbb{R}^d$ that involves the Wasserstein path-space UQ bound. Based on the WUP, \Cref{sec:SGM} establishes population-level $\W_1$ error bounds for SGMs and flow matching, while \Cref{sec:flow matching} develops corresponding finite-sample $\W_1$ error bounds. \Cref{sec:equivariant_SGM} specializes to $G$-invariant targets based on \Cref{sec:flow matching}: \Cref{sec:equivariant_bound} gives the bound under data augmentation, \Cref{sec:HJBproperty} develops the structural properties of equivariant score matching, and \Cref{sec:wrap} carries out the quantitative comparison between equivariant parametrization and data augmentation and \Cref{sec:equivariantFM} discusses equivariant flow matching. \Cref{sec:experiments} reports numerical corroboration. All proofs are collected in \Cref{sec:proofs} and the appendices, which also contain the sharper torus version of WUP via reflection coupling (\Cref{appendix:torus}). \Cref{sec:conclusion} concludes the paper with a discussion of related future directions and open questions.

\section{Background}\label{sec:background}
In this paper, we assume the domain is $\R^d$ and denote by $\CP(\R^d)$ the space of probability measures on $\R^d$ with respect to the Borel algebra. The Wasserstein-1 metric, denoted by $\mathcal{W}_1$, in its Kantorovich–Rubinstein duality, is defined as:
\begin{equation}
    \mathcal{W}_1(\pi_1,\pi_2) = \sup_{\gamma\in\text{Lip}_1(\R^d)}\left\{\E_{\pi_1}[\gamma]-\E_{\pi_2}[\gamma]\right\}
\end{equation}
for any $\pi_1$, $\pi_2\in\CP(\R^d)$, where $\text{Lip}_1(\R^d)$ is the class of 1-Lipschitz functions on $\R^d$; that is,
\begin{equation}
    \text{Lip}_1(\R^d)\coloneq \left\{\gamma:\R^d\to\R,\,  \abs{\gamma(x)-\gamma(y)}\leq \norm{x-y}_2, \forall x,y\in\R^d\right\}.
\end{equation}
In this paper, we study generative models governed by stochastic differential equations (SDEs). Given a drift term or a vector field $\mathbf{f}(x,t)$, we consider the following pair of forward and backward diffusion processes
\begin{align}
    \diff{x_t} &= -\mathbf{f}(x_t,t)\diff{t} + \boldsymbol{\sigma}(t)\diff{W_t},\quad x_0\sim \pi; \label{eq:forward}\\
    \diff{y_t} &= \left(\mathbf{f}(y_t,T-t)+\boldsymbol{\sigma}(T-t)\boldsymbol{\sigma}(T-t)^\top\nabla\log\eta^\pi(y_t,T-t)\right)\diff{t} + \boldsymbol{\sigma}(T-t)\diff{W_t},\quad y_0\sim m_0,\label{eq:backward}
\end{align}
where $\boldsymbol{\sigma}(t)$ is a $d\times d$ matrix whose Hilbert–Schmidt norm is square-integrable in time and $W_t$ is the $d$-dimensional standard Wiener process. We denote the law in \eqref{eq:forward} by $x_t\sim \eta^\pi(\cdot,t)$, where $\pi\in \CP(\R^d)$ is an unknown data distribution, which we only have access to through finite i.i.d. samples in practice. Here, $\nabla \log \eta^\pi(x,t)$ is called the \textit{score function}. It is known from \cite{anderson1982reverse} that if $m_0=\eta^\pi(\cdot,T)$, then we have $y_t \sim \eta^\pi(\cdot,T-t)$. Note that the corresponding probability measure of the process in \eqref{eq:forward} can be formulated as the solution of the Fokker-Planck equation
\begin{equation}\label{eq:general_FP}
    \partial_t \rho-\frac{1}{2}\Tr(\boldsymbol{\sigma}(t)\boldsymbol{\sigma}(t)^\top\nabla^2\rho)-\Div(\rho \mathbf{f})=0,\,\, \rho(0)=\pi,
\end{equation}
where $\Tr$ is the trace of a matrix and $\nabla^2\rho$ is the Hessian matrix of $\rho$ in variable $x$.

We list some notation that will be used in the rest of the paper. For square matrices $A$ and $B$, we write $A\succeq B$ if $A-B$ is positive semi-definite. Moreover, for a function or a probability density $m(x,t)$ that depends on both spatial and temporal variables, we occasionally hide the spatial variables and write $m(t)$ to simplify the notation. For a twice-differentiable time-evolving vector field $\mathbf{s}(x,t) = (s^1,s^2,\dots, s^d)$ defined on $\R^d\times[0,T]$, its $C^2$-norm is defined as $\norm{\mathbf{s}}_{C^2(\R^d\times[0,T])}\coloneqq \sup_{x\in\R^d,t\in[0,T]}\norm{\mathbf{s}}_2 + \sup_{x\in\R^d,t\in[0,T]}\norm{\nabla\mathbf{s}}_2 + \sup_{x\in\R^d,t\in[0,T]}\norm{\nabla^2\mathbf{s}}_{\text{F}}$, where $\norm{\mathbf{s}}_2$ is the vector $l^2$-norm of $\mathbf{s}$, $\norm{\nabla\mathbf{s}}_2$ is the matrix $2$-norm of the Jacobian of $\mathbf{s}$, and $\norm{\nabla^2\mathbf{s}}^2_{\text{F}}\coloneqq \sum_{i=1}^d\sum_{j,k=1}^d\left(\partial^2_{j,k}s^i\right)^2$.

\section{Wasserstein path-space divergences}\label{sec:Wdivergence}
We propose a new Wasserstein UQ bound in this section.
\begin{definition}\label{def:WUQ}
    Suppose $\rho_1^{[0,T]}$ and $\rho_2^{[0,T]}$ are the path laws of the following SDEs
    \begin{align*}
        \diff{x_t} &= \mathbf{b}_1(x_t,t)\diff{t} + \boldsymbol{\sigma}(t)\diff{W_t},\quad x_0\sim m_1;\\
        \diff{y_t} &= \mathbf{b}_2(y_t,t)\diff{t} + \boldsymbol{\sigma}(t)\diff{W_t},\quad y_0\sim m_2,
    \end{align*}
    from $t=0$ to $t=T$ respectively. Assume that $\boldsymbol{\sigma}(t)$ is measurable and its Hilbert–Schmidt norm is square-integrable in time and $\sup_{x\in\R^d, t\in[0,T]}\norm{\nabla \mathbf{b}_1(x,t)}_{2}<\infty$, then we define
    \begin{equation}\label{eq: WPD}
        \Xi\left(\rho_1^{[0,T]}\| \rho_2^{[0,T]}\right)\coloneqq \mathcal{W}_1(m_1,m_2)+\sqrt{T}\left(\int_0^T\int_{\R^d} \norm{(\mathbf{b}_2-\mathbf{b}_1)(x,t)}^2 \diff{\rho_2(x,t)}\diff{t}\right)^\frac{1}{2},
    \end{equation}
    where $\rho_2(x,t)$ is the marginal law of $\rho_2^{[0,T]}$ at time $t$, and $\boldsymbol{\sigma}(t)\boldsymbol{\sigma}(t)^\top$ is positive semidefinite, without requiring uniform ellipticity, and may be degenerate or identically zero.
\end{definition}
$\Xi$ is asymmetric, and we show below that $\Xi$ actually defines a divergence in the path-space.

\begin{theorem}[Path space divergence]\label{prop:pathspace}
    Under the assumption in \Cref{def:WUQ}, we have $\Xi\left(\rho_1^{[0,T]}\| \rho_2^{[0,T]}\right) \geq 0$ and $\Xi\left(\rho_1^{[0,T]}\| \rho_2^{[0,T]}\right) = 0$ if and only if $\rho_1^{[0,T]} = \rho_2^{[0,T]}$.
\end{theorem}
\begin{proof}
    The nonnegativity of $\Xi$ is obvious. First assume $\Xi\left(\rho_1^{[0,T]}\| \rho_2^{[0,T]}\right) = 0$, then we have $\mathcal{W}_1(m_2,m_1)=0$ and $\int_0^T\int_{\R^d} \norm{(\mathbf{b}_2-\mathbf{b}_1)(x,t)}^2 \diff{\rho_2(x,t)}\diff{t}=0$. The former implies $m_1 = m_2$, and the latter implies that $\mathbf{b}_1(x,t) = \mathbf{b}_2(x,t)$ for $\diff{\rho_2(x,t)}\diff{t}$-a.e. Together with the uniqueness in law for the SDE of $x_t$ implied by the Lipschitzness of $\mathbf{b}_1$, we have $\rho_1^{[0,T]} = \rho_2^{[0,T]}$. 
    
    Conversely, suppose $\rho_1^{[0,T]} = \rho_2^{[0,T]}$, then we immediately have their marginal laws at $t=0$ coincide; that is, $m_1 = m_2$ and thus $\W_1(m_1,m_2)=0$. Let $\rho^{[0,T]} = \rho_1^{[0,T]} = \rho_2^{[0,T]}$, and let $z_t$ denote the canonical process on path space: $z_t(\omega) = \omega(t)$. Since $\rho^{[0,T]}$ is the path law of the first SDE, $M_{1,t}\coloneqq z_t-z_0-\int_0^t\mathbf{b}_1(z_t,t)\diff{s}$ is a local martingale under $\rho^{[0,T]}$. Similarly, $M_{2,t}\coloneqq z_t-z_0-\int_0^t\mathbf{b}_2(z_t,t)\diff{s}$ is a local martingale under $\rho^{[0,T]}$. Therefore, their difference
    \begin{equation*}
        M_{1,t} - M_{2,t} = -\int_0^t\left(\mathbf{b}_1(z_t,t)-\mathbf{b}_2(z_t,t)\right)\diff{s}
    \end{equation*}
    is a local martingale. The right-hand side is an ordinary time integral, hence a continuous finite-variation process. It also starts from 0. Therefore, it is a continuous finite-variation local martingale starting from $0$, so it must be identically zero, $\forall t\in[0,T]$, $\rho^{[0,T]}$ a.s. Hence, $\mathbf{b}_1(z_t,t) = \mathbf{b}_2(z_t,t)$, $\rho^{[0,T]}\otimes\diff{t}$ a.s. Since $z_t\sim\diff{\rho_2}(x,t)$, under the common path law, this implies $\mathbf{b}_1(x,t) = \mathbf{b}_2(x,t)$, $\diff{\rho_2}(x,t)\diff{t}$-a.e. Therefore, $\int_0^T\int_{\R^d} \norm{(\mathbf{b}_2-\mathbf{b}_1)(x,t)}^2 \diff{\rho_2(x,t)}\diff{t}=0$. Together, we have proved $\Xi\left(\rho_1^{[0,T]}\| \rho_2^{[0,T]}\right) = 0$ when the path laws are equal.
\end{proof}
Note that \Cref{prop:pathspace} also applies to deterministic dynamics when $\boldsymbol{\sigma}(t)\equiv 0$. Thus, $\Xi$ can be cast as a UQ bound for the sensitivity analysis for both stochastic and deterministic dynamics. In contrast to the UQ bounds that involve the KL divergence \cite{dupuis2016path, zou2026goal}, our $\W_1$ metric makes such a UQ bound well-defined even for singular measures, such as discrete observables and measures with low-dimensional support.
\section{Wasserstein Uncertainty Propagation on $\R^d$}\label{sec:wup}
To see how the Wasserstein path-space divergence defined in \eqref{eq: WPD} is derived, we start by considering two forward diffusion processes that differ in the drift terms and initial probability measures but have the same diffusion coefficients that are independent of the spatial location $x$. To quantify the difference between the probability measures that evolve according to such diffusion processes, we provide the following $\mathcal{W}_1$ error bound that takes into account the weighted $L^2$ difference in the drift terms and the $\mathcal{W}_1$ distance between the initial measures. In contrast to Theorem 3.1 in \cite{mimikos2024score}, which relies on Bernstein-type estimates for Hamilton–Jacobi–Bellman equations and is therefore restricted to bounded domains, here we use a probabilistic approach that also applies to unbounded domains and Brownian motions with time-varying coefficients. More importantly, our analysis also includes the case when there is no noise at all, i.e., $\boldsymbol{\sigma}(t)\equiv 0$, while the Bernstein methods rely on noise, i.e., the existence of a Laplacian term in the PDE.
\begin{theorem}[Wasserstein Uncertainty Propagation]\label{thm:WUP}
    Let $\mathbf{b}_1,\mathbf{b}_2:\R^d\times[0,T]\to\R^d$ be differentiable vector fields and $m_1,m_2\in\CP(\R^d)$. Probability measures $\rho_i$ $(i=1,2)$ are given by
\begin{equation}\label{eq:WUP_1}
    \partial_t \rho_i-\frac{1}{2}\Tr(\boldsymbol{\sigma}(t)\boldsymbol{\sigma}(t)^\top\nabla^2\rho_i)-\Div(\rho_i \mathbf{b}_i)=0,\,\, \rho_i(0)=m_i.
\end{equation}
Then we have 
\begin{enumerate}
    \item If $c\coloneq\sup_{x\in\R^d, t\in[0,T]}\norm{\nabla \mathbf{b}_1(x,t)}_{2}<\infty$, then 
    \begin{equation}\label{eq:WUP1}
    \mathcal{W}_1(\rho_2(T),\rho_1(T))\leq e^{cT}\cdot\Xi\left(\rho_1^{[0,T]}\| \rho_2^{[0,T]}\right);
    \end{equation}
    
    \item If the negation of the drift term $\mathbf{b}_1(x,t) = \nabla U(x,t) + \mathbf{f}(x,t)$, where $U:\R^d\times[0,T]\to\R$, and both $\nabla^2 U\succeq c I_d$ and $\norm{\nabla\mathbf{f}(x,t)}_{2}\leq A$ hold uniformly in $x\in\R^d$ and $t\in [0,T]$ for some $A,c>0$, then
    \begin{equation}\label{eq:WUP2}
    \mathcal{W}_1(\rho_2(T),\rho_1(T))\leq e^{(-c+A)T}\cdot\Xi\left(\rho_1^{[0,T]}\| \rho_2^{[0,T]}\right).
    \end{equation}
\end{enumerate}
\end{theorem}
Note that the right-hand side of \eqref{eq:WUP1} and \eqref{eq:WUP2} are exactly the Wasserstein path-space divergence in \eqref{eq: WPD}, up to multiplicative factors in the time horizon $T$.
\begin{proof}
    We begin by proving the first part. Taking the difference of \eqref{eq:WUP_1} between $i=1,2$, the measure $\lambda=\rho_1-\rho_2$ satisfies the PDE
    \begin{equation}\label{eq:lambdaPDE_0}
        \partial_t\lambda-\frac{1}{2}\Tr(\boldsymbol{\sigma}(t)\boldsymbol{\sigma}(t)^\top\nabla^2\lambda) - \Div(\lambda \mathbf{b}_1+\rho_2(\mathbf{b}_1-\mathbf{b}_2))=0\,\,\text{in}\,\,\R^d\times(0,T),\quad \lambda(\cdot,0)=m_2-m_1.
    \end{equation}
    Let $\phi:\R^d\times[0,T]\to \R$ be a test function that evolves in time. We integrate in space and time against \eqref{eq:lambdaPDE_0} and apply integration by parts to obtain
    \begin{align}\label{eq:integrationbyparts_0}
        \int_{\R^d} \lambda(x,T)\phi(x,T) - \lambda(x,0)\phi(x,0)\diff{x} & + \int_0^T\int_{\R^d} \lambda(-\partial_t\phi-\frac{1}{2}\Tr(\boldsymbol{\sigma}(t)\boldsymbol{\sigma}(t)^\top\nabla^2\phi)+\mathbf{b}_1\cdot\nabla\phi)\diff{x}\diff{t}\\
        & + \int_0^T\int_{\R^d} \rho_2\nabla\phi\cdot(\mathbf{b}_1-\mathbf{b}_2)\diff{x}\diff{t}=0.\nonumber
    \end{align}
    Notice that if we choose the test function $\phi$ to satisfy the Kolmogorov backward equation
    \begin{equation}\label{eq:KBE_0}
    -\partial_t\phi-\frac{1}{2}\Tr(\boldsymbol{\sigma}(t)\boldsymbol{\sigma}(t)^\top\nabla^2\phi)+\mathbf{b}_1\cdot\nabla\phi=0\,\, \text{in}\,\,\R^d\times[0,T),\quad\phi(x,T)=\psi(x)\,\, \text{in}\,\,\R^d
    \end{equation} 
    with terminal condition $\psi\in\CF$ for some function class $\CF$, then from \eqref{eq:integrationbyparts_0}, we have 
    \begin{equation}\label{eq:difference_equation}
        \int_{\R^d} \lambda(x,T)\psi(x)\diff{x} = \int_{\R^d} \lambda(x,0)\phi(x,0)\diff{x} + \int_0^T\int_{\R^d} \rho_2(t)\nabla\phi(x,t)\cdot(\mathbf{b}_2-\mathbf{b}_1)(t)\diff{x}\diff{t}.
    \end{equation}
    We take $\CF$ to be the class of 1-Lipschitz functions on $\R^d$. Taking the supremum on both sides of \eqref{eq:difference_equation} over $\psi\in\CF$, we have 
    \begin{equation}\label{eq:WUP_bound}
        \mathcal{W}_1(\rho_2(T),\rho_1(T))\leq \sup_{\psi\in\CF}\abs{\int_{\R^d} \lambda(x,0)\phi(x,0)\diff{x}} + \sup_{\psi\in\CF}\abs{\int_0^T\int_{\R^d} \rho_2\nabla\phi\cdot(\mathbf{b}_2-\mathbf{b}_1)\diff{x}\diff{t}}.
    \end{equation}
    The first term on the right-hand side of \eqref{eq:WUP_bound} can be bounded by $e^{\norm{\nabla \mathbf{b}_1}T}\mathcal{W}_1(m_1,m_2)$, since $\phi(x,0)$ is $e^{\norm{\nabla \mathbf{b}_1}T}$-Lipschitz by \Cref{lemma:gradient_bound_0}, and the second term can be bounded by 
    \[\sqrt{T}\norm{\nabla\phi}\norm{\mathbf{b}_2-\mathbf{b}_1}_{L^2(\rho_2)}\leq \sqrt{T} e^{\norm{\nabla \mathbf{b}_1}T}\norm{\mathbf{b}_2-\mathbf{b}_1}_{L^2(\rho_2)},
    \]
    since we can split $\abs{\rho_2\nabla\phi\cdot(\mathbf{b}_2-\mathbf{b}_1)}$ into $\sqrt{\rho_2}\norm{\nabla\phi}$ and $\sqrt{\rho_2}\norm{\mathbf{b}_2-\mathbf{b}_1}$ and apply the Cauchy-Schwarz inequality and notice that the integral of $\rho_2$ over $\R^d$ is 1.

    For the second part of the theorem, note that in this case we can apply \Cref{lemma:gradient_bound} to bound $\norm{\nabla\phi}$, and the claim follows.
\end{proof}
\begin{remark}
    There is no constraint on the initial probability measures $m_i$ in \Cref{thm:WUP}: they can be empirical, singular, or supported on a low-dimensional manifold.
\end{remark}
\begin{remark}
    Part 2 of \Cref{thm:WUP} suggests that if $\mathbf{b}_1$ is the gradient of a strongly convex potential and the error $\epsilon$ grows polynomially in $T$, then $\lim_{T\to\infty}\mathcal{W}_1(\rho_2(T),\rho_1(T))=0$. It can be understood intuitively that the structure of $b_1$ can ``overcome'' the error in the drift term.
\end{remark}
\begin{remark}\label{remark:decay}
    When the domain is compact, for example, a $d$-dimensional torus $R\mathbb{T}^d$, we can employ a non-synchronous coupling, such as the reflection coupling in \cite{eberle2016reflection}, to derive a time-exponentially decaying Lipschitz bound for \Cref{lemma:gradient_bound_0}. This allows the factors $e^{cT}$ in \eqref{eq:WUP1} to be replaced by an exponentially decaying term in $T$, thereby yielding a stronger result than \cite[Theorem 3.1]{mimikos2024score}. We prove this improved bound in \Cref{appendix:torus}. However, the reflection coupling cannot improve our bound in the unbounded domain $\R^d$ unless additional structure, such as contractivity, is assumed on $\mathbf{b}_1$.
\end{remark}
\begin{remark}
    \Cref{thm:WUP} offers a different perspective from Lemma 22 in the work on stochastic interpolants \cite{albergo2025stochastic} for quantifying discrepancies between ground truth and observables via the Girsanov-type result and KL divergence. In particular, our framework accounts not only for differences in the drift terms, but also for discrepancies between the initial measures. More importantly, \Cref{thm:WUP} remains applicable in the noiseless setting, such as flow matching with $\boldsymbol{\sigma}(t)\equiv 0$. In this case, the Feynman–Kac formula used in the proofs of \Cref{lemma:gradient_bound_0} and \Cref{lemma:gradient_bound} reduces to the method of characteristics, whereas \cite{albergo2025stochastic} requires a nondegenerate, albeit small, diffusion coefficient in front of the Laplacian. Furthermore, the $\W_1$ metric employed in \Cref{thm:WUP} is well-defined for finite samples and measures supported on low-dimensional manifolds, while the KL divergence requires absolute continuity.
\end{remark}
We provide estimations of the Lipschitz constant of the test function that appears in the proof of \Cref{thm:WUP}. Instead of applying the logarithm transform and performing Bernstein-type gradient estimations as in \cite{mimikos2024score}, we bound the gradient of the test function that solves \eqref{eq:KBE_0} using probabilistic formulations for the PDE solution.
\begin{lemma}\label{lemma:gradient_bound_0}
Suppose $\phi$ is the solution of the backward equation
\begin{equation}\label{eq:KBE_general_0}
    \partial_t\phi+\frac{1}{2}\Tr(\boldsymbol{\sigma}(t)\boldsymbol{\sigma}(t)^\top\nabla^2\phi)+\mathbf{b}\cdot\nabla\phi-a\phi=0\,\, \text{in}\,\,\R^d\times[0,T),\quad\phi(x,T)=\psi(x)\,\, \text{in}\,\,\R^d,
    \end{equation} 
    where $a\in\R$ is a constant, $\norm{\nabla \mathbf{b}(x,t)}_{2}\leq c$ uniformly in $x\in\R^d$ and $t\in[0,T)$, and $\psi$ is $L$-Lipschitz. Then $\phi(x,t)$ is $e^{(c-a)(T-t)}L$-Lipschitz for any $t\in[0,T)$. In particular, $\phi(x,t)$ is also $L$-Lipschitz for $t\in[0,T)$ when $a\geq c$.
\end{lemma}
\begin{proof}
    Using the Feynman-Kac formula, the solution of \eqref{eq:KBE_general_0} can be expressed as
    \begin{equation}
        \phi(x,t) = \E\left[e^{-\int_t^T a\diff{\tau}}\psi(x_T)\mid x_t=x\right] = e^{-a(T-t)}\E\left[\psi(x_T)\mid x_t=x\right],
    \end{equation}
    where $\diff{x_t} = \mathbf{b}(x_t,t)\diff{t}+ \boldsymbol{\sigma}(t)\diff{W_t}$. Let $\diff{y_t} = \mathbf{b}(y_t,t)\diff{t}+ \boldsymbol{\sigma}(t)\diff{W_t}$, where the increment of Brownian motion $\diff{W_t}$ is coupled with that of $\diff{x_t}$, and we have
    \begin{equation}\label{eq:F-K_0}
        \phi(x,t)-\phi(y,t) =  e^{-a(T-t)}\E\left[\psi(x_T)-\psi(y_T)\mid x_t=x,y_t=y\right],
    \end{equation}
    where the conditional expectation is taken jointly on $x_T$ and $y_T$. Since $\psi$ is $L$-Lipschitz, $\psi(x_T)-\psi(y_T)\leq L\norm{x_T-y_T}_2$.
    Thus, we have 
    \begin{align*}
        \frac{\diff{}}{\diff{t}}\norm{x_t-y_t}_2^2 
        &= 2\inner{x_t-y_t}{\frac{\diff{(x_t-y_t)}}{\diff{t}}}\\
        &= 2\inner{x_t-y_t}{\mathbf{b}(x_t,t)- \mathbf{b}(y_t,t)}\\
        &\leq 2c\norm{x_t-y_t}_2^2.
    \end{align*}
    Then, we have by the Gr\"onwall's inequality,
    \begin{equation}
        \norm{x_T-y_T}^2 
        \leq \norm{x-y}^2 e^{\int_t^T 2c \diff{\tau}} = e^{2c(T-t)}\norm{x-y}^2.
    \end{equation}
    Therefore, 
    \begin{equation}
        \psi(x_T)-\psi(y_T)\leq L\norm{x_T-y_T}\leq e^{c(T-t)}L\norm{x-y},
    \end{equation}
    and returning to \eqref{eq:F-K_0}, we have 
    \begin{equation}
        \phi(x,t) - \phi(y,t)\leq e^{-a(T-t)}e^{c(T-t)}L\norm{x-y} = e^{(c-a)(T-t)}L\norm{x-y},
    \end{equation}
    whenever $T\geq t$. Switching $x$ and $y$, we can also show that $\phi(y,t) - \phi(x,t)\leq e^{(c-a)(T-t)}L\norm{x-y}$, and therefore $\phi(x,t)$ is $e^{(c-a)(T-t)}L$-Lipschitz for $t\leq T$.
\end{proof}

\begin{lemma}\label{lemma:gradient_bound}
Suppose $\phi$ is the solution of the backward equation
\begin{equation}\label{eq:KBE_general}
    \partial_t\phi+\frac{1}{2}\Tr(\boldsymbol{\sigma}(t)\boldsymbol{\sigma}(t)^\top\nabla^2\phi)+\mathbf{b}\cdot\nabla\phi-a\phi=0\,\, \text{in}\,\,\R^d\times[0,T),\quad\phi(x,T)=\psi(x)\,\, \text{in}\,\,\R^d,
    \end{equation} 
    where $a\in\R$ is a constant, and $\mathbf{b}(x,t) = -\nabla U(x,t) + \mathbf{f}(x,t)$, where $U:\R^d\times[0,T]\to\R$, and both $\nabla^2 U\succeq c I_d$ and $\norm{\nabla\mathbf{f}(x,t)}_2\leq A$ hold uniformly in $x\in\R^d$ and $t\in [0,T]$ for some $A,c>0$, and $\psi$ is an $L$-Lipschitz function. Then  $\phi(\cdot,t)$ is $e^{-(a+c-A)(T-t)}L$-Lipschitz, and in particular, it is $L$-Lipschitz for $t\in[0,T)$ if $a+c-A\geq 0$.
\end{lemma}
\begin{proof}
    Using the Feynman-Kac formula, the solution of \eqref{eq:KBE_general} can be expressed as
    \begin{equation}
        \phi(x,t) = \E\left[e^{-\int_t^T a\diff{\tau}}\psi(x_T)\mid x_t=x\right] = e^{-a(T-t)}\E\left[\psi(x_T)\mid x_t=x\right],
    \end{equation}
    where $\diff{x_t} = \mathbf{b}(x_t,t)\diff{t}+ \boldsymbol{\sigma}(t)\diff{W_t}$. For $\diff{y_t} = \mathbf{b}(y_t,t)\diff{t}+ \boldsymbol{\sigma}(t)\diff{W_t}$, where the increment of Brownian motion $\diff{W_t}$ is coupled with that of $x_t$, we have
    \begin{equation}\label{eq:F-K}
        \phi(x,t)-\phi(y,t) =  e^{-a(T-t)}\E\left[\psi(x_T)-\psi(y_T)\mid x_t=x,y_t=y\right],
    \end{equation}
    where the conditional expectation is taken jointly on $x_T$ and $y_T$. Since $\psi$ is $L$-Lipschitz, $\psi(x_T)-\psi(y_T)\leq L\norm{x_T-y_T}$. Moreover, $\diff{(x_t-y_t)} = \left[-\left(\nabla U(x_t)-\nabla U(y_t)\right)+\left(\mathbf{f}(x_t,t)-\mathbf{f}(y_t,t)\right)\right]\diff{t}$, and by the strong convexity of $U$, we have 
    \begin{equation*}
        \inner{\nabla U(x)-\nabla U(y)}{x-y}\geq c\norm{x-y}^2.
    \end{equation*}
    Thus, we have 
    \begin{align*}
        \frac{\diff{}}{\diff{t}}\norm{x_t-y_t}^2 
        &= 2\inner{x_t-y_t}{\frac{\diff{(x_t-y_t)}}{\diff{t}}}\\
        &= 2\inner{x_t-y_t}{-\left(\nabla U(x_t)-\nabla U(y_t)\right)}+2\inner{x_t-y_t}{\mathbf{f}(x_t,t)-\mathbf{f}(y_t,t)}\\
        &\leq -2c \norm{x_t-y_t}^2 + 2A \norm{x_t-y_t}^2.
    \end{align*}
    Then, we have by the Gr\"onwall's inequality,
    \begin{equation}
        \norm{x_T-y_T}^2 
        \leq \norm{x-y}^2 e^{\int_t^T 2(A-c) \diff{\tau}} = e^{2(A-c)(T-t)}\norm{x-y}^2.
    \end{equation}
    Therefore, 
    \begin{equation}
        \psi(x_T)-\psi(y_T)\leq L\norm{x_T-y_T}\leq e^{(A-c)(T-t)}L\norm{x-y},
    \end{equation}
    and using \eqref{eq:F-K}, we have 
    \begin{equation}
        \phi(x,t) - \phi(y,t)\leq e^{-a(T-t)}e^{(A-c)(T-t)}L\norm{x-y} = e^{-(a+c-A)(T-t)}L\norm{x-y},
    \end{equation}
    whenever $T\geq t$. Switching $x$ and $y$, we can also show that $\phi(y,t) - \phi(x,t)\leq e^{-(a+c-A)(T-t)}L\norm{x-y}$, and therefore $\phi(x,t)$ is $e^{-(a+c-A)(T-t)}L$-Lipschitz for $t\leq T$.
\end{proof}

As a result of \Cref{lemma:gradient_bound}, we obtain the exponential decay of the $\mathcal{W}_1$ distance between two probability measures that follow the same SDE: $\diff{x_t} = -\alpha x_t\diff{t}+\boldsymbol{\sigma}(t)\diff{W_t}$, where $\alpha>0$ is a constant.

\begin{corollary}[Exponential decay of $\mathcal{W}_1$]\label{lemma:exponential_decay} 
Let $m_1,m_2\in\CP(\R^d)$. If $\rho_i$ for $i=1,2$ are given by
\begin{equation}\label{eq:Fokker-Planck}
    \partial_t \rho_i-\frac{1}{2}\Tr(\boldsymbol{\sigma}(t)\boldsymbol{\sigma}(t)^\top\nabla^2\rho_i)-\Div(\alpha x\cdot \rho_i)=0,\,\, \rho_i(0)=m_i,
\end{equation}
    where $\alpha>0$ is a constant, then we have 
    \begin{equation}
        \mathcal{W}_1(\rho_1(t),\rho_2(t))\leq e^{-\alpha t}\mathcal{W}_1(m_1,m_2).
    \end{equation}
\end{corollary}
\begin{proof}
    Let $\CF$ be the class of 1-Lipschitz functions on $\R^d$. Then for any $\psi\in\CF$, by \eqref{eq:difference_equation}, we have
    \begin{equation}
        \int_{\R^d} \psi(x)\diff{
        \left(\rho_1(t)-\rho_2(t)\right)}
        = \int_{\R^d} \phi(x,0) \diff{
        \left(m_1(t)-m_2(t)\right)},
    \end{equation}
    since we simply take $\mathbf{b}_1=\mathbf{b}_2=\alpha x$ in \eqref{eq:difference_equation} for this case. Note that by taking $a=0$ and $c=\alpha$ in \Cref{lemma:gradient_bound} for the KBE in \eqref{eq:KBE_0}, $\phi(\cdot,0)$ is $e^{-\alpha t}$-Lipschitz. Taking the supremum over $\psi\in\CF$, we have
    \begin{equation}
        \mathcal{W}_1(\rho_1(t),\rho_2(t))\leq e^{-\alpha t}\mathcal{W}_1(m_1,m_2).
    \end{equation}
\end{proof}

\section{Population-level $\W_1$ error bounds for flow-based models on $\R^d$}\label{sec:SGM}
In this section, we apply \Cref{thm:WUP} to derive population-level $\W_1$ error bounds for both SGMs and flow matching, quantifying the discrepancy between the target data distribution and the corresponding generated distributions. We always denote by $\pi\in\CP(\R^d)$ the target data measure that the generative model seeks to approximate. 
\subsection{Population-level bounds for score-based generative models}
In SGMs, the generated measure is $m(T)$, where $m(t)$ follows the time-reversal (denoising) diffusion process \eqref{eq:backward} with the score function $\nabla\log\eta^\pi$ replaced by its score approximation $\mathbf{s}_\theta$. In particular, we consider the forward process in \eqref{eq:forward} as the Ornstein–Uhlenbeck (OU) process
\begin{equation}\label{SDE:OU}
\diff{x_t} = - x_t\diff{t}+\sqrt{2}\diff{W_t}.
\end{equation}
Then $m(t)$ satisfies the following backward PDE
\begin{equation}\label{eq:KBE2}
    \partial_t m=-\Div\left(( x+2\mathbf{b}_\theta)\cdot m\right)+\Delta m\,\, \text{in}\,\,\R^d\times(0,T),\quad m(0) =\mathcal{N}(0,I_d)\,\, \text{in}\,\,\R^d,
\end{equation}
where $\mathbf{b}_\theta(x,t) = \mathbf{s}_\theta(x,T-t)$, and in practice, $\mathbf{s}_\theta$ is often parametrized by neural networks.

The score approximation $\mathbf{s}_\theta$ is typically obtained via score matching, which is commonly achieved by optimizing parameterized vector fields with respect to a chosen score matching objective function. For example, the \textit{denoising score matching} (DSM) \cite{vincent2011connection} objective is defined as:
\begin{equation}\label{eq:DSM}
    \CJ_D(\eta^\pi,\mathbf{s}_\theta)=\int_0^T\int_{\R^d}\int_{\R^d} \norm{\mathbf{s}_\theta-\nabla\log\eta^{x'}}^2\diff{\eta^{x'}}(s)\diff{\pi}(x')\diff{s},
\end{equation}
where $\eta^{x'}(s)$ denotes the conditional probability from $x'$ at time 0 to $x$ of \eqref{eq:forward} at time $s$. In addition, we also introduce two other variants of score matching objectives. The \textit{explicit score matching} (ESM) objective \cite{song2020score} is defined as: 
\begin{equation}\label{eq:ESM}
    \CJ_E(\eta^\pi,\mathbf{s}_\theta)=\int_0^T\int_{\R^d} \norm{\mathbf{s}_\theta-\nabla\log\eta^\pi}^2\diff{\eta^\pi}(s)\diff{s},
\end{equation}
and it is obvious that $\CJ_E(\eta^\pi,\mathbf{s}_\theta) = \CJ_D(\eta^\pi,\mathbf{s}_\theta)$ by the law of total expectation. The \textit{implicit score matching} (ISM) objective \cite{song2020sliced} is defined as:
\begin{equation}\label{eq:ISM}
    \CJ_I(\eta^\pi,\mathbf{s}_\theta)\coloneq \int_0^T\int_{\R^d} \left(\norm{\mathbf{s}_\theta}^2+2\Div(\mathbf{s}_\theta)\right)\diff{\eta^\pi(s)}\diff{s},
\end{equation}
which may be more practical for score matching. By an expansion of the square of the norm, it is easy to verify that $\CJ_D(\rho,\mathbf{s}_\theta) = \CJ_E(\rho,\mathbf{s}_\theta) =\CJ_I(\rho,\mathbf{s}_\theta) + 4\norm{\nabla \sqrt{\rho}}_{L^2(\R^d\times[0,T])}^2$ for any $\rho\in\CP(\R^d)\times[0,T]$. This suggests that the optimal solutions to the DSM, ESM and ISM coincide for the same target probability measure.

We provide a $\W_1$ error analysis for SGMs at the population level.
\begin{theorem}[$\mathcal{W}_1$ error bound for SGMs, population level]\label{thm:generalization_error_continuos}
Let $\pi$ be the target data measure and assume the noising process follows \eqref{SDE:OU}. Let $A>0$, and suppose the generated distribution $m(T)$ satisfies \eqref{eq:KBE2}. Moreover, assume $\sup_{x\in\R^d,t\in[0,T]}\norm{\nabla_x \mathbf{s}_\theta(x,t)}_2\leq A$ and let $e_{nn} = \CJ_D(\eta^{\pi},\mathbf{s}_\theta)$ be the DSM error. Then we have
\begin{equation}
    \begin{aligned}
    \mathcal{W}_1(\pi,m(T))\leq  e^{(1+2A)T}\left(e^{-T}\mathcal{W}_1\left(\pi,\mathcal{N}(0,I_d)\right) + \sqrt{T\cdot e_{nn}}\right).
    \end{aligned}
\end{equation}
\end{theorem}
The $\mathcal{W}_1$ bound in \Cref{thm:generalization_error_continuos} remains well-defined and meaningful even when the target distribution $\pi$ does not have a density; for example, $\pi$ is supported on a low-dimensional manifold. Moreover, $\mathcal{W}_1$ naturally yields a sample complexity bound that we will prove in \Cref{thm:generalization_error}.
\begin{proof}
    Recall that $\eta^{\pi}$ satisfies
    \begin{equation}\label{eq:OUevolution_continous}
        \begin{cases}
            \partial_t \eta^{\pi}-\Delta \eta^{\pi} -\Div( x\cdot\eta^{\pi})=0 \,\, \text{in}\,\,\R^d\times(0,T],\\
            \eta^{\pi}(0) = \pi \,\,\text{in}\,\,\R^d.
        \end{cases}
    \end{equation}
    We define the drift
\begin{equation*}
    \mathbf{b}^{\pi}(x,t) \coloneq \nabla\log\eta^{\pi}(x,T-t)
\end{equation*}
and let $m^\pi(x,t) = \eta^{\pi}(x,T-t)$ that satisfies the backward equation
\begin{equation}\label{eq:KBE3_continuous}
    \begin{cases}
        \partial_t m^\pi=\Delta m^\pi -  \Div\left(m^\pi( x+2\mathbf{b}^{\pi})\right),\\
        m^\pi(0) = \eta^{\pi}(T).
    \end{cases}
\end{equation}
Applying the first part of \cref{thm:WUP} to \eqref{eq:KBE2} and \eqref{eq:KBE3_continuous}, we have 
\begin{equation}
    \mathcal{W}_1\left(\pi,m(T)\right) 
    = \mathcal{W}_1\left(m^\pi(T),m(T)\right)\leq e^{(1+2A) T}\left(\mathcal{W}_1\left(m^{\pi}(0),\mathcal{N}(0,I_d)\right)+\sqrt{T}\norm{\mathbf{b}^{\pi}-\mathbf{b}_\theta}_{L^2(m^\pi)}\right).
\end{equation}
By \Cref{lemma:exponential_decay}, we have 
\begin{align*}
    \mathcal{W}_1(m^{\pi}(0),\mathcal{N}(0,I_d)) = \mathcal{W}_1(\eta^{\pi}(T),\mathcal{N}(0,I_d))\leq e^{- T}\mathcal{W}_1(\pi,\mathcal{N}(0,I_d)),
\end{align*}
since the standard Gaussian $\mathcal{N}(0, I_d)$ is the stationary measure of the forward OU process. 
\end{proof}
\begin{remark}
    The exponential factor $e^{(1+2A)T}$ originates from the first part of \Cref{thm:WUP}, which is established on the unbounded domain $\R^d$. In general, the reverse-time dynamics may amplify small perturbations of the terminal distribution. Consequently, without additional structural assumptions on $\pi$ (e.g., log-concavity), the initialization error can only be controlled through a Lipschitz stability estimate, resulting in the factor $e^{(1+2A)T}$. By contrast, on a bounded domain, this growth can be replaced by an exponentially decaying factor in $T$; see \Cref{remark:decay}. 
\end{remark}

\subsection{Population-level bounds for flow matching}
In flow matching generative models \cite{lipman2022flow}, the forward and backward processes are deterministic ODEs rather than SDEs and, in most cases, we wish to approximate the velocity field $\mathbf{v}(x,t)$, such that at the particle level, we have
\begin{equation}\label{eq:FM_particle}
    \frac{\diff{x}}{\diff{t}}=\mathbf{v}(x,t),\quad x(0)\sim\pi, \quad x(1)\sim\CN(0,I_d).
\end{equation}
Hence, in order to generate data, we would run the ODE backward in time from the standard Gaussian using the velocity $-\mathbf{v}$. The choice of $\mathbf{v}$ can vary, such as the average velocity over $x(1)\sim\CN(0,I_d)$ of the linear interpolation with a constant speed; see \cite{lipman2022flow,albergo2025stochastic} for examples. 

The particle movement \eqref{eq:FM_particle} corresponds to the continuity equation
\begin{equation}\label{eq:transport}
    \partial_t\rho + \nabla\cdot(\rho\mathbf{v})=0,\quad \rho(0)=\pi, \quad \rho(1)=\CN(0,I_d).
\end{equation}
The flow matching objective for a velocity approximation $\mathbf{s}_\theta$, similar to the ESM for SGMs, is defined as
\begin{equation}\label{eq:FM}
    \CJ_{FM}(\mathbf{v},\mathbf{s}_\theta,\rho)=\int_0^1\int_{\R^d} \norm{\mathbf{v}-\mathbf{s}_\theta}^2\diff{\rho}(s)\diff{s},
\end{equation}
where $\mathbf{v}$ and $\rho$ satisfy \eqref{eq:transport}.
Using the approximation $\mathbf{b}_\theta(x,t) = \mathbf{s}_\theta(x,1-t)$, the generation follows the equation
\begin{equation}\label{eq:FW_generation}
    \partial_t m - \nabla\cdot\left(m\mathbf{b}_\theta\right)=0,\quad m(0)=\CN(0,I_d),
\end{equation}
and $m(1)$ is the generated distribution of the flow matching model. Then we have the following error analysis for flow matching generative models at the population level.

\begin{theorem}[$\mathcal{W}_1$ error bound for flow matching, population level]\label{thm:generalization_error_FM_continuous}
Assume $\mathbf{s}_\theta$ satisfies $\norm{\nabla_x \mathbf{s}_\theta(x,t)}_2\leq A$ uniformly over $(x,t)\in\R^d\times[0,1]$ for some constant $A>0$ and let $e_{nn} = \CJ_{FM}(\mathbf{v},\mathbf{s}_\theta,\rho)$. Then we have
\begin{equation}
\mathcal{W}_1(\pi,m(1))\leq  e^{A}\sqrt{e_{nn}}.
\end{equation}
\end{theorem}
\begin{proof}
Note that $\pi$ is the solution of the time reversal of \eqref{eq:transport} at $t=1$. Thus, by a direct application of the first part of \Cref{thm:WUP}, we readily have
\begin{equation}
    \mathcal{W}_1(\pi,m(1))\leq e^A \sqrt{e_{nn}},
\end{equation}
since there is no mismatch for the initial measures, both of which are the standard Gaussian.
\end{proof}
\begin{remark}
    The bound in \Cref{thm:generalization_error_FM_continuous} for flow matching is derived using the new Wasserstein path-space divergence, and it cannot be derived via Girsanov’s theorem, since the dynamics are deterministic and contain no diffusion term. While \cite[Proposition 3]{albergobuilding} establishes a related guarantee in the Wasserstein-2 metric for deterministic flows, their analysis does not directly extend to SGMs. In contrast, our adjoint formulation provides a unified framework for both stochastic and deterministic flows. In particular, this perspective highlights that generative flows can learn low-dimensional manifold structures, on which the Wasserstein distance remains well-defined whereas the KL divergence may fail to be finite, in both stochastic and deterministic settings under a common theoretical framework.
\end{remark}
\begin{remark}
    \cite{kunkel2026distribution} constructs a velocity field $\mathbf{v}$ satisfying $\norm{\nabla_x \mathbf{v}}_2\leq A^\star$ uniformly over $(x,t)\in\R^d\times[0,1]$ for some $A^\star >0$. Therefore, by choosing a neural network architecture for $\mathbf{s}_\theta$ with $A\geq A^\star$, $e_{nn}$ can, in principle, be made arbitrarily small by the universal approximation property of neural networks.
\end{remark}

\section{Finite-sample $\W_1$ error bounds for flow-based models on $\R^d$}\label{sec:flow matching}
In this section, we investigate the finite-sample performance of SGMs and flow matching on $\R^d$. While the results of the previous section characterize the population-level $\W_1$ approximation error relative to the target distribution $\pi$, practical training is based only on finitely many samples from $\pi$. We therefore analyze the additional statistical error introduced by replacing population objectives with their empirical counterparts and derive finite-sample $\W_1$ error bounds for the resulting learned generative models.
\subsection{Finite-sample bounds for score-based generative models}
We assume that, as in practical applications, the target distribution $\pi$ is observed only through $N$ i.i.d. samples $\{z_i\}_{i=1}^N$ drawn from $\pi$. The score matching or the DSM objective \eqref{eq:DSM} is then approximated such that $\eta^\pi(x,t)$ is replaced by its kernel density estimator $\eta^N(x,t)= \int P_t(x,y)\diff{\pi^N}(y)$, where $\eta^N(x,0)= \pi^N \coloneq \frac{1}{N}\sum_{i=1}^N\delta_{z_i}$ is the empirical measure, and $P_t(x,y)$ is the transition kernel of the OU process defined in \eqref{eq:transition_probability}. Since the kernel estimator does not have a well-defined score function at $t = 0$, the DSM objective is often integrated over $t \in [\epsilon, \epsilon+T]$ rather than starting from $t=0$, an example of early stopping in SGMs \cite{song2020score}. Specifically, this is equivalent to score matching for the mollified distribution $\eta^{N,\epsilon}(x,t) \coloneq \int P_t(x,y)\diff{\eta^N(y,\epsilon)}$ for $0\leq t\leq T$.

For the error analysis for SGMs in the finite-sample regime, we assume that the target data measure $\pi$ is sub-Gaussian; that is, $\pi$ satisfies the following lemma.
\begin{assumption}\label{assumption:tail}
    There exists $a>0$, such that
    \begin{equation}
        \int_{\R^d}e^{a\norm{x}^2}\diff{\pi(x)}<\infty.
    \end{equation}
\end{assumption}
For the proof of \Cref{thm:generalization_error}, we define the following quantity for $\pi$:
\begin{equation}
    R_\pi(\epsilon)\coloneq\int_{\R^d}\frac{\int_{\R^d}e^{-\frac{\norm{x-y}^2}{\epsilon}}\diff{\pi(y)}}{\int_{\R^d}e^{-\frac{\norm{x-y}^2}{2\epsilon}}\diff{\pi(y)}}\diff{x}.
\end{equation}
Then, we have the following lemma.
\begin{lemma}\label{lemma:finite_early_stopping}
    Suppose that $\pi$ satisfies \Cref{assumption:tail}, then there exists $\epsilon(\pi)>0$, such that $R_\pi(\epsilon)<\infty$ for any $\epsilon\geq\epsilon(\pi)$. Moreover, we have $\lim_{\epsilon\to+\infty}\epsilon^{-d/2}R_\pi(\epsilon) = (2\pi)^{d/2}$.
\end{lemma}
\begin{proof}
    Let $M_1(x) = \int_{\R^d}e^{-\frac{\norm{x-y}^2}{\epsilon}}\diff{\pi(y)}$ and $M_2(x) = \int_{\R^d}e^{-\frac{\norm{x-y}^2}{2\epsilon}}\diff{\pi(y)}$ be the numerator and the denominator of the integrand. Then we have
    \begin{align*}
        M_1(x) = e^{-\frac{\norm{x}^2}{\epsilon}}\int_{\R^d}e^{\frac{2x\cdot y}{\epsilon}-\frac{\norm{y}^2}{\epsilon}}\diff{\pi(y)}&\leq e^{-\frac{\norm{x}^2}{\epsilon}}\int_{\R^d}e^{\frac{2x\cdot y}{\epsilon}}\diff{\pi(y)}\\
        &\leq e^{-\frac{\norm{x}^2}{\epsilon}}\int_{\R^d}e^{\frac{2x\cdot y}{\epsilon}}\diff{\pi(y)}\\
        &\leq e^{-\frac{\norm{x}^2}{\epsilon} + \frac{\norm{x}^2}{a\epsilon^2}}\int_{\R^d}e^{a\norm{y}^2}\diff{\pi(y)},
    \end{align*}
    where in the last inequality we use the Cauchy-Schwartz inequality $\frac{2x\cdot y}{\epsilon} = \leq a\norm{y}^2 + \frac{\norm{x}^2}{a\epsilon^2}$.

    To provide a lower bound for $M_2(x)$, we can pick $R>0$ such that $p_R\coloneqq \pi\left(B(0,R)\right)>0$. Then, we have
    \begin{align*}
        M_2(x)&\geq \int_{\norm{y}\leq R}e^{-\frac{\norm{x-y}^2}{2\epsilon}}\diff{\pi(y)}\\
        &\geq \int_{\norm{y}\leq R}e^{-\frac{\norm{x}^2}{2\epsilon}-\frac{R\norm{x}}{\epsilon}-\frac{R^2}{2\epsilon}}\diff{\pi(y)}\\
        &= p_R e^{-\frac{\norm{x}^2}{2\epsilon}-\frac{R\norm{x}}{\epsilon}-\frac{R^2}{2\epsilon}},
    \end{align*}
    where in the last inequality we use the fact that $\norm{x-y}^2\leq\norm{x}^2+2R\norm{x}+R^2$ for $\norm{y}\leq R$.

    Therefore,
    \begin{equation}\label{eq:M1_M2}
        \frac{M_1(x)}{M_2(x)}\leq \frac{\int_{\R^d}e^{a\norm{y}^2}\diff{\pi(y)}}{p_R}e^{\frac{R^2}{2\epsilon}}\cdot e^{-\frac{\norm{x}^2}{2\epsilon}+ \frac{\norm{x}^2}{a\epsilon^2}+\frac{R\norm{x}}{\epsilon}},
    \end{equation}
    and the integral of the right-hand side in $x$ is finite whenever $2\epsilon<a\epsilon^2$; that is, $\epsilon>\frac{2}{a}$, where $a$ is the constant in \Cref{assumption:tail}.

    For the second part, we can write 
    \begin{equation}
    R_\pi(\epsilon) = \int_{\R^d} \frac{M_1(x)}{M_2(x)}\diff{x} = \epsilon^{d/2}\int_{\R^d} \frac{M_1(\sqrt{\epsilon}z)}{M_2(\sqrt{\epsilon}z)}\diff{z} \coloneqq \epsilon^{d/2}\int_{\R^d} F_\epsilon(z)\diff{z},
    \end{equation}
    and $\epsilon^{-d/2}R_\pi(\epsilon) = \int_{\R^d}F_\epsilon(z)\diff{z}$. Note that 
    \begin{align*}
        M_1(\sqrt{\epsilon}z) = \int_{\R^d}e^{-\frac{\norm{\sqrt{\epsilon}z-y}^2}{\epsilon}}\diff{\pi(y)} = e^{-\norm{z}^2}\int_{\R^d}e^{\frac{2z\cdot y}{\sqrt{\epsilon}}-\frac{\norm{y}^2}{\epsilon}}\diff{\pi(y)},
    \end{align*}
    the integrand of which can be bounded by
    \begin{equation}
        e^{\frac{2z\cdot y}{\sqrt{\epsilon}}-\frac{\norm{y}^2}{\epsilon}}\leq e^{\frac{2z\cdot y}{\sqrt{\epsilon}}} \leq a\norm{y}^2 + \frac{\norm{z}^2}{a\epsilon}\leq a\norm{y}^2 + \frac{\norm{z}^2}{a}
    \end{equation}
    by the Cauchy-Schwartz inequality we used for the first part for $\epsilon>1$. Thus, by the dominated convergence theorem, $\lim_{\epsilon\to\infty} M_1(\sqrt{\epsilon}z) = e^{-\norm{z}^2}$. Similarly, we can show that $\lim_{\epsilon\to\infty} M_2(\sqrt{\epsilon}z) = e^{-\frac{\norm{z}^2}{2}}$, and so that we have $\lim_{\epsilon\to\infty} F_\epsilon(z) = e^{-\frac{\norm{z}^2}{2}}$, and $\int_{\R^d} \lim_{\epsilon\to\infty}F_\epsilon(z)\diff{z}= (2\pi)^{d/2}$.
    Moreover, using \eqref{eq:M1_M2}, we have
    \begin{align*}
        F_\epsilon(z) = \frac{M_1(\sqrt{\epsilon}z)}{M_2(\sqrt{\epsilon}z)} &\leq \frac{\int_{\R^d}e^{a\norm{y}^2}\diff{\pi(y)}}{p_R}e^{\frac{R^2}{2\epsilon}}\cdot e^{-\frac{\norm{z}^2}{2}+ \frac{\norm{z}^2}{a\epsilon}+\frac{R\norm{z}}{\sqrt{\epsilon}}}\\
        &\leq \frac{\int_{\R^d}e^{a\norm{y}^2}\diff{\pi(y)}}{p_R}e^{\frac{R^2}{2\epsilon}}\cdot e^{-\frac{\norm{z}^2}{4}+R\norm{z}}
    \end{align*}
    when $\epsilon>\max\{4a^{-1}, 1\}$, and the right-hand side is integrable in $z$. Therefore, by the dominated convergence theorem, $\lim_{\epsilon\to\infty}\epsilon^{-d/2}R_\pi(\epsilon) = \lim_{\epsilon\to\infty}\int_{\R^d}F_\epsilon(z)\diff{z} = \int_{\R^d} \lim_{\epsilon\to\infty}F_\epsilon(z)\diff{z} = (2\pi)^{d/2}$.
\end{proof}
The following lemma provides a bound on the population-level error resulting from early stopping.
\begin{lemma}\label{lemma:d1_pi_pi_epsilon}
    Let $\Gamma_1(\pi)$ be the first moment of $\pi$, and $\pi^\epsilon$ be the law of running the process \eqref{eq:forward} for $\epsilon$ time with $\mathbf{f}(x_t,t) = x_t$ and $\sigma(t)=\sqrt{2}$. Then for any $\epsilon>0$, we have
    \begin{equation}
        \mathcal{W}_1(\pi^\epsilon,\pi)\leq C_d\sqrt{(1-e^{-2\epsilon})}+ (1-e^{-\epsilon})\Gamma_1(\pi),
    \end{equation}
    where $C_d>0$ is a dimensional constant.
\end{lemma}
\begin{proof}
    We can formulate $\pi^\epsilon$ as
    \begin{equation}\label{eq:pi_epsilon}
        \pi^\epsilon(x) = \int_{\R^d} P_\epsilon(x,y)\diff{\pi(y)},
    \end{equation}
    where 
    \begin{equation}\label{eq:transition_probability}
    P_\epsilon(x,y) = [2\pi(1-e^{-2\epsilon})]^{-\frac{d}{2}}e^{-\frac{\norm{x-e^{-\epsilon}y}^2}{2(1-e^{-2\epsilon})}}
    \end{equation}
    is the transition probability of the OU process.
    Then, for any 1-Lipschitz function $g$ we have
    \begin{align*}
        \int_{\R^d}g(x)\diff{(\pi^\epsilon-\pi)}&=\int_{\R^d}g(x)\diff{\pi^\epsilon(x)}-\int_{\R^d}g(x)\diff{\pi(x)}\\
        &= \int_{\R^d}\int_{\R^d}g(x)P_\epsilon(x,y)\diff{\pi(y)}\diff{x} - \int_{\R^d}g(x)\diff{\pi(x)}\\
        &= \int_{\R^d}\int_{\R^d}g(y)P_\epsilon(y,x)\diff{y}\diff{\pi(x)} - \int_{\R^d}g(x)\diff{\pi(x)}\\
        &= \int_{\R^d}\left(\int_{\R^d}g(y)P_\epsilon(y,x)\diff{y} - g(x)\right)\diff{\pi(x)},
    \end{align*}
    where in the third equality we apply a change of variables $x\mapsto y$ and $y\mapsto x$, and Fubini's theorem guarantees the switch of the order of integration. To bound the integrand, we have
    \begin{align*}
        \abs{\int_{\R^d}g(y)P_\epsilon(y,x)\diff{y} - g(x)}
        &= \abs{\int_{\R^d}[2\pi(1-e^{-2\epsilon})]^{-\frac{d}{2}}e^{-\frac{\norm{y-e^{-\epsilon}x}^2}{2(1-e^{-2\epsilon})}}g(y)\diff{y} - g(x)}\\
        &= \abs{\int_{\R^d}[2\pi(1-e^{-2\epsilon})]^{-\frac{d}{2}}e^{-\frac{\norm{y}^2}{2(1-e^{-2\epsilon})}}g(y+e^{-\epsilon}x)\diff{y} - g(x)}\\
        &= \abs{\int_{\R^d}[2\pi(1-e^{-2\epsilon})]^{-\frac{d}{2}}e^{-\frac{\norm{y}^2}{2(1-e^{-2\epsilon})}}\left[g(y+e^{-\epsilon}x)-g(x)\right]\diff{y}}\\
        &\leq \int_{\R^d}[2\pi(1-e^{-2\epsilon})]^{-\frac{d}{2}}e^{-\frac{\norm{y}^2}{2(1-e^{-2\epsilon})}}\abs{g(y+e^{-\epsilon}x)-g(x)}\diff{y}\\
        &\leq \int_{\R^d}[2\pi(1-e^{-2\epsilon})]^{-\frac{d}{2}}e^{-\frac{\norm{y}^2}{2(1-e^{-2\epsilon})}}\left(\norm{y} +(1-e^{-\epsilon})\norm{x}\right)\diff{y}\\
        &= C_d\sqrt{(1-e^{-2\epsilon})}+(1-e^{-\epsilon})\norm{x}.
    \end{align*}
    Therefore, we have 
    \begin{align*}
        \abs{\int_{\R^d}g(x)\diff{(\pi^\epsilon-\pi)}}
        &\leq \int_{\R^d}\abs{\int_{\R^d}g(y)P_\epsilon(y,x)\diff{y} - g(x)}\diff{\pi(x)}\\
        &\leq \int_{\R^d} C_d\sqrt{(1-e^{-2\epsilon})}+(1-e^{-\epsilon})\norm{x}\diff{\pi(x)}\\
        &= C_d\sqrt{(1-e^{-2\epsilon})}+ (1-e^{-\epsilon})\Gamma_1(\pi).
    \end{align*}
\end{proof}
We now derive an average $\mathcal{W}_1$ error bound for SGMs when score matching is performed using empirical measures with early stopping $\epsilon>0$; that is, the DSM objective is $\CJ_D(\eta^{N,\epsilon},\mathbf{s}_\theta)$ instead of $\CJ_D(\eta^{\pi},\mathbf{s}_\theta)$. Compared with the bound in \Cref{thm:generalization_error_continuos}, additional error terms arise due to finite-sample effects and early stopping.
\begin{theorem}[$\mathcal{W}_1$ error bound for SGMs, finite-sample regime]\label{thm:generalization_error}
Assume that the target data measure $\pi$ satisfies \Cref{assumption:tail} and that the noising process follows \eqref{SDE:OU}, and the early stopping $\epsilon>\epsilon(\pi)$ defined in \Cref{lemma:finite_early_stopping}. Let $A>0$, and suppose the generated distribution $m(T)$ satisfies \eqref{eq:KBE2}. Moreover, assume $\sup_{x\in\R^d,t\in[0,T]}\norm{\nabla_x \mathbf{s}_\theta(x,t)}_2\leq A$ and let $e_{nn} = \CJ_D(\eta^{N,\epsilon},\mathbf{s}_\theta)$ be the DSM error with early stopping $\epsilon$ for any $N$ i.i.d. samples $\pi^N$ drawn from $\pi$. Then we have
\begin{equation}
    \begin{aligned}
    \E\left[\mathcal{W}_1(\pi,m(T))\right]\leq C_d\sqrt{1-e^{-2\epsilon}}&+ (1-e^{-\epsilon})\Gamma_1(\pi)\\ &+ e^{(1+2A)T}\left(e^{-(T+\epsilon)}\mathcal{W}_1\left(\pi,\mathcal{N}(0,I_d)\right)+\sqrt{T}\sqrt{e_{nn}'}\right),
    \end{aligned}
\end{equation}
where
\begin{equation}
    \begin{aligned}
    e_{nn}'= \E[e_{nn}]&+C_d T e^{-\epsilon}\E\left[\mathcal{W}_1(\pi^{N},\pi)\right]\norm{\mathbf{s}_\theta}^2_{C^2(\R^d\times[0,T])}\\ 
    &+ \frac{4}{ N}\left([2\pi(e^{2(\epsilon+T)}-1)]^{-\frac{d}{2}} R_{\pi}(e^{2(\epsilon+T)}-1)-1\right),
    \end{aligned}
\end{equation}
the expectation is taken over $N$ i.i.d. samples drawn from $\pi$, and the $C_d$'s appear above can be different dimensional constants.
\end{theorem}
Before we prove the theorem, we interpret each term in the upper bound:
\begin{itemize}
\item The term $C_d\sqrt{(1-e^{-2\epsilon})}+ (1-e^{-\epsilon})\Gamma_1(\pi)$ is due to early stopping $\epsilon$;

\item $e^{(1+2A)T}e^{-(T+\epsilon)}\mathcal{W}_1\left(\pi,\mathcal{N}(0,I_d)\right)$ has exponential growth in time horizon $T$. The factor $e^{(1+2A)T}$ results from the unbounded domain. 
There is a trade-off between $A$ and $e_{nn}$, which the DSM error with $N$ training samples and early stopping $\epsilon$.  Moreover, by \Cref{lemma:score_decomposition}, the worst bound is $\E[e_{nn}]\leq \E\int_{0}^T\int_{\R^d}\norm{\nabla_x\log\eta^{N,\epsilon}}^2\eta^{N,\epsilon}\diff{x}\diff{t}\leq \int_0^T\frac{4e^{-2(\epsilon+t)}}{(1-e^{-2(\epsilon+t)})^2}\Gamma_2(\pi) + \frac{2d}{1-e^{-2(\epsilon+t)}}\diff{t} = 2(\frac{1}{1-e^{-2\epsilon}}-\frac{1}{1-e^{-2(\epsilon+T)}}) \Gamma_2(\pi) + d\log\frac{e^{2(\epsilon+T)}-1}{e^{2\epsilon}-1}$, which approaches infinity when the early stopping $\epsilon\to 0$.

\item Within $e_{nn}'$, $C_d T e^{-\epsilon}\E[\mathcal{W}_1(\pi,\pi^{N})]\norm{\mathbf{s}_\theta}^2_{C^2(\R^d\times[0,T])}$ includes the $C^2$-norm of the score approximation $\mathbf{s}_\theta$ with early stopping $\epsilon$ and $A^2\leq\norm{\mathbf{s}_\theta}^2_{C^2(\R^d\times[0,T])}$ is assumed to be bounded for all $\theta$, and the error of the empirical estimations of $\mathcal{W}_1$ whose convergence rates have been provided in \cite{fournier2015rate}, for example, $\E[\mathcal{W}_1(\pi,\pi^{N})]\lesssim N^{-1/d}$ for $d\geq 3$; the last term of $e_{nn}'$ has a decay rate $N^{-1}$ that decays faster than $\E[\mathcal{W}_1(\pi,\pi^{N})]$ in $N$; the asymptotics that $R_\pi(e^{2(\epsilon+T)}-1)$ is of the order $e^{d(\epsilon+T)}$ when $T\to\infty$ has been provided in \Cref{lemma:finite_early_stopping}, and thus the coefficient of the $N^{-1}$ term remains bounded when $T\to\infty$.

\item When the domain is $R\mathbb{T}^d$, the $d$-dimensional torus endowed with periodic boundary conditions, the WUP result from \Cref{appendix:torus} applies. As a consequence, the growth factor $e^{(1+2A)T}$ can be replaced by 
an exponentially decaying factor $C_2e^{-C_1T}$, where $C_1 = \frac{2}{D^2}e^{-(1+2A)D^2/8}$, $C_2 = 2e^{(1+2A)D^2/8}$ and $D=\pi R\sqrt{d}$, leading to a substantially sharper error bound than that of \cite{mimikos2024score}.
\end{itemize}
\begin{proof}
    We decompose $\mathcal{W}_1(\pi,m(T))$ as
\begin{equation}\label{ieq:oracle}
    \mathcal{W}_1(\pi,m(T))\leq \mathcal{W}_1(\pi,\pi^\epsilon)+\mathcal{W}_1(\pi^\epsilon,m(T)).
\end{equation}
By \Cref{lemma:d1_pi_pi_epsilon}, we have $\mathcal{W}_1(\pi,\pi^\epsilon)\leq C_d\sqrt{(1-e^{-2\epsilon})}+ (1-e^{-\epsilon})\Gamma_1(\pi)$.
To bound the second term on the right-hand side of \eqref{ieq:oracle}, we define $\eta^{\pi,\epsilon}:\R^d\times [0,T]\to\R$ as the solution of
\begin{equation}\label{eq:OUevolution}
    \begin{cases}
        \partial_t \eta^{\pi,\epsilon}-\Delta \eta^{\pi,\epsilon} -\Div( x\cdot\eta^{\pi,\epsilon})=0 \,\, \text{in}\,\,\R^d\times[0,T],\\
        \eta^{\pi,\epsilon}(0) = \pi^\epsilon \,\,\text{in}\,\,\R^d.
    \end{cases}
\end{equation}
Moreover, we define the drift
\begin{equation*}
    \mathbf{b}^{\pi,\epsilon}(x,t) \coloneq \nabla\log(\eta^{\pi,\epsilon})(x,T-t)
\end{equation*}
and let $m^{\epsilon}(x,t) = \eta^{\pi,\epsilon}(x,T-t)$ which satisfies the backward equation
\begin{equation}\label{eq:KBE3}
    \begin{cases}
        \partial_t m^{\epsilon}=\Delta m^{\epsilon} -  \Div\left(m^{\epsilon}( x+2\mathbf{b}^{\pi,\epsilon})\right),\\
        m^{\epsilon}(0) = \eta^{\pi,\epsilon}(T).
    \end{cases}
\end{equation}
Then applying the first part of \cref{thm:WUP} to \eqref{eq:KBE2} and \eqref{eq:KBE3}, we have 
\begin{equation}
    \mathcal{W}_1\left(\pi^\epsilon,m(T)\right) 
    = \mathcal{W}_1\left(m^\epsilon(T),m(T)\right)\leq e^{(1+2A) T}\left(\mathcal{W}_1\left(m^{\epsilon}(0),\mathcal{N}(0,I_d)\right)+\sqrt{T}\norm{\mathbf{b}^{\pi,\epsilon}-\mathbf{b}_\theta}_{L^2(m^\epsilon)}\right).
\end{equation}
By \Cref{lemma:exponential_decay}, we have 
\begin{align*}
    \mathcal{W}_1(m^{\epsilon}(0),\mathcal{N}(0,I_d)) = \mathcal{W}_1(\eta^{\pi,\epsilon}(T),\mathcal{N}(0,I_d))\leq e^{- T}\mathcal{W}_1(\pi^{\epsilon},\mathcal{N}(0,I_d))\leq e^{-(T+\epsilon)}\mathcal{W}_1(\pi,\mathcal{N}(0,I_d)),
\end{align*}
since the standard Gaussian $\mathcal{N}(0, I_d)$ is the stationary measure of the forward OU process. 

To bound $\norm{\mathbf{b}^{\pi,\epsilon}-\mathbf{b}_\theta}_{L^2(m^\epsilon)}$, we first note that, by a change of variable for time, $\norm{\mathbf{b}^{\pi,\epsilon}-\mathbf{b}_\theta}_{L^2(m^\epsilon)} = \sqrt{\CJ_D(\eta^{\pi,\epsilon},\mathbf{s}_\theta)}$. Also, recall that $\CJ_D(\eta^{\pi,\epsilon},\mathbf{s}_\theta) = \CJ_I(\eta^{\pi,\epsilon},\mathbf{s}_\theta) + 4\norm{\nabla \sqrt{\eta^{\pi,\epsilon}}}_2^2$. Thus, to utilize the DSM error on empirical measures, $\CJ_D(\eta^{N,\epsilon},\mathbf{s}_\theta)$, which is a random variable, we make the following decomposition 
    \begin{align*}
        \CJ_D(\eta^{\pi,\epsilon},\mathbf{s}_\theta) & = \CJ_D(\eta^{N,\epsilon},\mathbf{s}_\theta) + \left(\CJ_I(\eta^{\pi,\epsilon},\mathbf{s}_\theta) - \CJ_I(\eta^{N,\epsilon},\mathbf{s}_\theta)\right) + 4\left(\norm{\nabla \sqrt{\eta^{\pi,\epsilon}}}_2^2 - \norm{\nabla \sqrt{\eta^{N,\epsilon}}}_2^2\right).
    \end{align*}
    By \Cref{lemma:DSM_1}, we have 
    \begin{equation}    \CJ_I(\eta^{\pi,\epsilon},\mathbf{s}_\theta) - \CJ_I(\eta^{N,\epsilon},\mathbf{s}_\theta) \leq C_d T \mathcal{W}_1(\pi^{\epsilon},\pi^{N,\epsilon})\norm{\mathbf{s}_\theta}^2_{C^2(\R^d\times[0,T])}\leq C_d T e^{-\epsilon}\mathcal{W}_1(\pi,\pi^{N})\norm{\mathbf{s}_\theta}^2_{C^2(\R^d\times[0,T])},
    \end{equation}
    where the last inequality is due to \Cref{lemma:exponential_decay}. By \Cref{lemma:DSM_2}, we have 
    \begin{align*}
        &\norm{\nabla \sqrt{\eta^{\pi,\epsilon}}}_2^2 - \norm{\nabla \sqrt{\eta^{N,\epsilon}}}_2^2\\
        &= \int_{\R^d} \pi^\epsilon\log(\pi^\epsilon) - \pi^{N,\epsilon}\log(\pi^{N,\epsilon}) \diff{x} + \int_{\R^d} \eta^{N,\epsilon}(T)\log(\eta^{N,\epsilon}(T))  -\eta^{\pi,\epsilon}(T)\log(\eta^{\pi,\epsilon}(T))\diff{x},
    \end{align*}
    In particular, by Jensen's inequality, we have 
    \begin{equation}
        \E \int_{\R^d} - \pi^{N,\epsilon}\log(\pi^{N,\epsilon}) \diff{x}\leq \int_{\R^d} -\pi^\epsilon\log(\pi^\epsilon) \diff{x},
    \end{equation}
    since the differential entropy is a concave functional of probability measures. Thus, we have 
    \begin{equation}
    \E\left[\int_{\R^d} \pi^\epsilon\log(\pi^\epsilon) - \pi^{N,\epsilon}\log(\pi^{N,\epsilon}) \diff{x}\right]\leq 0.
    \end{equation}
    Moreover, we have
    \begin{align*}
        &\int_{\R^d} \eta^{N,\epsilon}(T)\log\left(\eta^{N,\epsilon}(T)\right)  -\eta^{\pi,\epsilon}(T)\log\left(\eta^{\pi,\epsilon}(T)\right)\diff{x}\\
        &=  \int_{\R^d} \eta^{N,\epsilon}(T)\log\left(\eta^{N,\epsilon}(T)\right)  -\eta^{N,\epsilon}(T)\log\left(\eta^{\pi,\epsilon}(T)\right)\diff{x} + \int_{\R^d} \eta^{N,\epsilon}(T)\log\left(\eta^{\pi,\epsilon}(T)\right)  -\eta^{\pi,\epsilon}(T)\log\left(\eta^{\pi,\epsilon}(T)\right)\diff{x}\\
        &= D_{\text{KL}}\left(\eta^{N,\epsilon}(T)\|\eta^{\pi,\epsilon}(T)\right) + \int_{\R^d} \left(\eta^{N,\epsilon}(T)  -\eta^{\pi,\epsilon}(T)\right)\log\left(\eta^{\pi,\epsilon}(T)\right)\diff{x},
    \end{align*}
    where $D_{\text{KL}}(p\|q)\coloneq \int\log\frac{p}{q}\diff{p}$ is the KL divergence between $p$ and $q$. Note that \[\E\int_{\R^d} \left(\eta^{N,\epsilon}(T)  -\eta^{\pi,\epsilon}(T)\right)\log(\eta^{\pi,\epsilon}(T))\diff{x}=0,\] and for the KL term, we have
    \begin{align*}
        D_{\text{KL}}\left(\eta^{N,\epsilon}(T)\|\eta^{\pi,\epsilon}(T)\right)&\leq \chi^2\left(\eta^{N,\epsilon}(T)\|\eta^{\pi,\epsilon}(T)\right)\coloneq \int_{\R^d}\frac{\left(\eta^{N,\epsilon}(T)-\eta^{\pi,\epsilon}(T)\right)^2}{\eta^{\pi,\epsilon}(T)}\diff{x}
    \end{align*}
    as it is a well-known fact that the KL divergence can be bounded from above by the chi-square distance $\chi^2$. Hence we have
    \begin{align*}
        \E\left[ D_{\text{KL}}\left(\eta^{N,\epsilon}(T)\|\eta^{\pi,\epsilon}(T)\right)\right]&\leq \int_{\R^d}\frac{\E\left(\eta^{N,\epsilon}(T)-\eta^{\pi,\epsilon}(T)\right)^2}{\eta^{\pi,\epsilon}(T)}\diff{x}\\
        &= \frac{1}{N}\int_{\R^d}\frac{\E\left(\eta^{1,\epsilon}(T)-\eta^{\pi,\epsilon}(T)\right)^2}{\eta^{\pi,\epsilon}(T)}\diff{x}\\
        &= \frac{1}{N}\left(\int_{\R^d}\frac{\E\left(\eta^{1,\epsilon}(T)\right)^2}{\eta^{\pi,\epsilon}(T)}\diff{x}-1\right).
    \end{align*}
    Furthermore, we have
    \begin{align*}
        \int_{\R^d}\frac{\E\left(\eta^{1,\epsilon}(T)\right)^2}{\eta^{\pi,\epsilon}(T)}\diff{x}
        &=\int_{\R^d}\frac{\int_{\R^d}P_{\epsilon+T}^2(x,y)\diff{\pi}(y)}{\int_{\R^d}P_{\epsilon+T}(x,y)\diff{\pi}(y)}\diff{x}\\
        &= [2\pi(1-e^{-2(\epsilon+T)})]^{-\frac{d}{2}}\int_{\R^d}\frac{\int e^{-\frac{\norm{x-e^{-(\epsilon+T)}y}^2}{(1-e^{-2(\epsilon+T)})}}\diff{\pi}(y)}{\int e^{-\frac{\norm{x-e^{-(\epsilon+T)}y}^2}{2(1-e^{-2(\epsilon+T)})}}\diff{\pi}(y)}\diff{x}\\
        &= e^{- d(\epsilon+T)}[2\pi(1-e^{-2(\epsilon+T)})]^{-\frac{d}{2}}\int_{\R^d}\frac{\int e^{-\frac{\norm{e^{-(\epsilon+T)} x-e^{-(\epsilon+T)}y}^2}{(1-e^{-2(\epsilon+T)})}}\diff{\pi}(y)}{\int e^{-\frac{\norm{e^{-(\epsilon+T)} x-e^{-(\epsilon+T)}y}^2}{2(1-e^{-2(\epsilon+T)})}}\diff{\pi}(y)}\diff{x}\\
        &= e^{- d(\epsilon+T)}[2\pi(1-e^{-2(\epsilon+T)})]^{-\frac{d}{2}}\int_{\R^d}\frac{\int e^{-\frac{\norm{x-y}^2}{(e^{2(\epsilon+T)}-1)}}\diff{\pi}(y)}{\int e^{-\frac{\norm{x-y}^2}{2(e^{2(\epsilon+T)}-1)}}\diff{\pi}(y)}\diff{x}\\
        &= [2\pi(e^{2(\epsilon+T)}-1)]^{-\frac{d}{2}} R_{\pi}(e^{2(\epsilon+T)}-1)
    \end{align*}
    where $P_{\epsilon+T}(x,y)$ is the transition probability defined in \eqref{eq:transition_probability} and $R_\pi$ is defined in \Cref{assumption:tail}.

    Combining all the estimates above together, we have
    \begin{align*}
        \E\left[\mathcal{W}_1(\pi,m(T))\right]
        &\leq C_d\sqrt{(1-e^{-2\epsilon})}+ (1-e^{-\epsilon})\Gamma_1(\pi)\\
        &\quad+ e^{(1+2A)T}\left(e^{-(T+\epsilon)}\mathcal{W}_1(\pi,\mathcal{N}(0,I_d))+\E\sqrt{T}\norm{\mathbf{b}^{\pi,\epsilon}-\mathbf{b}_\theta}_{L^2(m^\epsilon)}\right)\\
        &\leq C_d\sqrt{(1-e^{-2\epsilon})}+ (1-e^{-\epsilon})\Gamma_1(\pi) + e^{(1+2A)T}e^{-(T+\epsilon)}\mathcal{W}_1(\pi,\mathcal{N}(0,I_d))\\
        &\quad + e^{(1+2A) T}\sqrt{T}\cdot\E\sqrt{\CJ_D(\eta^{\pi,\epsilon},\mathbf{s}_\theta)}\\
        &\leq C_d\sqrt{(1-e^{-2\epsilon})}+ (1-e^{-\epsilon})\Gamma_1(\pi) + e^{(1+2A)T}e^{-(T+\epsilon)}\mathcal{W}_1(\pi,\mathcal{N}(0,I_d))\\
        &\quad + e^{(1+2A)T}\sqrt{T}\cdot\sqrt{\E\CJ_D(\eta^{\pi,\epsilon},\mathbf{s}_\theta)}\\
        &\leq C_d\sqrt{(1-e^{-2\epsilon})}+ (1-e^{-\epsilon})\Gamma_1(\pi) + e^{(1+2A)T}e^{-(T+\epsilon)}\mathcal{W}_1(\pi,\mathcal{N}(0,I_d))\\
        &\quad + e^{(1+2A)T}\sqrt{T}\cdot\sqrt{e_{nn}'},
    \end{align*}
    where 
    \begin{align*}
    e_{nn}'&= e_{nn}+C_d T e^{-\epsilon}\E[\mathcal{W}_1(\pi,\pi^{N})]\norm{\mathbf{s}_\theta}^2_{C^2(\R^d\times[0,T])}\\
    &\quad+ \frac{4}{ N}\left([2\pi(e^{2(\epsilon+T)}-1)]^{-\frac{d}{2}} R_{\pi}(e^{2(\epsilon+T)}-1)-1\right).
    \end{align*}    
\end{proof}
It follows from \eqref{eq:KBE2} that the drift term of the backward process is $x+2\mathbf{s}_\theta(x,T-t)$, and inspired by the empirical score decomposition proved in \Cref{lemma:score_decomposition} in \Cref{appendix:score_decomposition}, we can approximate the empirical score by 
\begin{equation}\label{eq:score_decomposition}
\mathbf{s}_\theta(x,t) = -[1-e^{-2(\epsilon+t)}]^{-1}x+\mathbf{v}_\theta(x,t),    
\end{equation}
where $\mathbf{v}_\theta$ approximates the residual. This decomposition is inspired by \cite[Proposition 1]{shi2024simplified}. Therefore, we can instead impose $\sup_{x\in\R^d,t\in[0,T]}\norm{\nabla_x \mathbf{v}_\theta(x,t)}_2\leq A$, and by \cref{thm:WUP}, the factor $e^{(1+2A)T}$ can be improved to $e^{(-1+2A)T}$ by noting that $[1-e^{-2(\epsilon+t)}]^{-1}>1$ or more precisely to $e^{(1+2A)T-\int_{0}^T2[1-e^{-2(\epsilon+t)}]^{-1}\diff{t}}$, which decays exponentially in $T$ when $A$ is small. The following corollary details this observation.
\begin{corollary}\label{cor:SGM_decomposition}
Assume that the target data measure $\pi$ satisfies \Cref{assumption:tail} and that the noising process follows \eqref{SDE:OU}, and the early stopping $\epsilon>\epsilon(\pi)$ defined in \Cref{lemma:finite_early_stopping}. Let $A>0$, and suppose the generated distribution $m(T)$ satisfies \eqref{eq:KBE2}. Moreover, assume $\mathbf{s}_\theta(x,t) = -[1-e^{-2(\epsilon+t)}]^{-1}x+\mathbf{v}_\theta(x,t)$, where $\sup_{x\in\R^d,t\in[0,T]}\norm{\nabla_x \mathbf{v}_\theta(x,t)}_2\leq A$ and let $e_{nn} = \CJ_D(\eta^{N,\epsilon},\mathbf{s}_\theta)$ for any $N$ i.i.d. samples $\pi^N$ drawn from $\pi$. Then we have
\begin{equation}
    \begin{aligned}
    \E\left[\mathcal{W}_1(\pi,m(T))\right]\leq C_d\sqrt{1-e^{-2\epsilon}}&+ (1-e^{-\epsilon})\Gamma_1(\pi)\\ &+ e^{(1+2A)T-2\int_{0}^T[1-e^{-2(\epsilon+t)}]^{-1}\diff{t}}\left(e^{-(T+\epsilon)}\mathcal{W}_1\left(\pi,\mathcal{N}(0,I_d)\right)+\sqrt{T}\sqrt{e_{nn}'}\right),
    \end{aligned}
\end{equation}
where
\begin{equation}
    \begin{aligned}
    e_{nn}'= \E[e_{nn}]&+C_d T e^{-\epsilon}\E\left[\mathcal{W}_1(\pi^{N},\pi)\right]\norm{\mathbf{s}_\theta}^2_{C^2(\R^d\times[0,T])}\\ 
    &+ \frac{4}{ N}\left([2\pi(e^{2(\epsilon+T)}-1)]^{-\frac{d}{2}} R_{\pi}(e^{2(\epsilon+T)}-1)-1\right),
    \end{aligned}
\end{equation}
the expectation is taken over $N$ i.i.d. samples drawn from $\pi$, and the $C_d$'s appear above can be different dimensional constants.     
\end{corollary}
\begin{proof}
    Since now the drift term for \eqref{eq:KBE2} becomes $x + 2\mathbf{s}_\theta(x,T-t) = x -2[1-e^{-2(\epsilon+t)}]^{-1}x+2\mathbf{v}_\theta(x,T-t)$, we can apply the second part of \cref{thm:WUP}, and the rest of the proof follows that of \Cref{thm:generalization_error}.
\end{proof}
By approximating the residual, the worst bound for score matching is $\E[e_{nn}]\leq \int_0^T\frac{e^{-2(\epsilon+t)}}{(1-e^{-2(\epsilon+t)})^2}\Gamma_2(\pi)\diff{t} = \frac{1}{2}(\frac{1}{1-e^{-2\epsilon}}-\frac{1}{1-e^{-2(\epsilon+T)}}) \Gamma_2(\pi)$, and is bounded due to \Cref{lemma:score_decomposition} in \Cref{appendix:score_decomposition}. As in the proof of \Cref{thm:generalization_error_continuos}, the population score function of $\eta^\pi$ can be decomposed so as to isolate the term $-[1-e^{-2t}]^{-1}x$ which generates contractivity. Nevertheless, $\int_0^T[1-e^{-2t}]^{-1}\diff{t}$ diverges at $t=0$, making the resulting estimate unbounded without early stopping. This suggests that, even at the population level, SGMs require early stopping unless additional assumptions such as log-concavity of $\pi$ are imposed.

\subsection{Finite-sample bounds for flow matching}
In the finite-sample regime, the velocity field $\tilde{\mathbf{v}}$ is designed such that
\begin{equation}\label{eq:transport_empirical}
    \partial_t\tilde{\rho} + \nabla\cdot(\tilde{\rho}\tilde{\mathbf{v}})=0,\quad \tilde{\rho}(0)=\pi^N, \quad \tilde{\rho}(1)=\CN(0,I_d),
\end{equation}
where $\pi^N$ is the empirical measure of $N$ i.i.d. samples drawn from $\pi$. Then we have the following direct bound that accounts for finite samples.
\begin{corollary}[$\mathcal{W}_1$ error bound for flow matching, finite-sample regime]\label{thm:generalization_error_FM}
Assume $\mathbf{s}_\theta$ satisfies $\norm{\nabla_x \mathbf{s}_\theta(x,t)}_2\leq A$ uniformly over $(x,t)\in\R^d\times[0,1]$ for some constant $A>0$ and for any $N$ i.i.d. samples $\pi^N$ drawn from $\pi$, let $e_{nn} = \CJ_{FM}(\tilde{\mathbf{v}},\mathbf{s}_\theta,\tilde{\rho})$. Then we have
\begin{equation}
\E\left[\mathcal{W}_1(\pi,m(1))\right]\leq  \E\left[\mathcal{W}_1(\pi,\pi^N)\right] + e^{A}\sqrt{\E[e_{nn}]},
\end{equation}
where the expectation is taken over $N$ i.i.d. samples drawn from $\pi$.
\end{corollary}
\begin{proof}
First, we have
\begin{equation}
    \mathcal{W}_1(\pi,m(1)) \leq \mathcal{W}_1(\pi,\pi^N) + \mathcal{W}_1(\pi^N,m(1)).
\end{equation}
Note that $\pi^N$ is the solution of the time reversal of \eqref{eq:transport_empirical} at $t=1$, and again by a direct application of the first part of \Cref{thm:WUP}, we have
\begin{equation}
    \mathcal{W}_1(\pi^N,m(1))\leq e^{A}\norm{\tilde{\mathbf{v}} - \mathbf{s}_\theta}_{L^2(\tilde{\rho})}\leq e^A \sqrt{e_{nn}}.
\end{equation}
\end{proof}
\begin{remark}[Lipschitzness of vector fields]
In contrast to \Cref{thm:generalization_error} for SGMs, the velocity field in flow matching is not intrinsically determined by the underlying distribution through a score function. Consequently, the quantity $\E[e_{nn}]$ depends not only on the neural network architecture but also on the particular choice of interpolation flow $\tilde{\mathbf v}$. Moreover, unlike the empirical score field, which admits the population counterpart, an empirical velocity field does not generally possess a canonical population-level analogue. As a result, the behavior of $\E[e_{nn}]$ may be substantially more sensitive to finite-sample effects. Furthermore, for many flow matching constructions, the target velocity field associated with the empirical measure $\pi^N$ may become increasingly irregular near the endpoints of the time interval. Therefore, under the architectural constraint $\sup_{(x,t)\in\mathbb{R}^d\times[0,1]}\|\nabla_x\mathbf{s}_\theta(x,t)\|_2\leq A$, the universal approximation theorem does not necessarily imply that $e_{nn}$ can be made arbitrarily small. Indeed, if the target velocity field cannot be well approximated within the prescribed Lipschitz class, then a nontrivial approximation error may remain. From this perspective, imposing a uniform Lipschitz bound can be viewed as an implicit regularization that may prevent exact memorization of the empirical measure while improving stability and generalization. It is also applicable in the SGM case \Cref{thm:generalization_error}, and it becomes even more obvious in \eqref{eq:score_decomposition}; in that sense,  the parametrization \eqref{eq:score_decomposition} takes out the singularity at $t=0, \epsilon=0$ and makes a more meaningful regularization assumption in \Cref{cor:SGM_decomposition}.
\end{remark}

\section{Provable efficiency of equivariant SGMs and equivariant flow matching}\label{sec:equivariant_SGM}
In this section, building on \Cref{sec:flow matching}, we consider generative modeling when the target distribution is invariant under a given group action. After introducing the necessary notions of group actions and symmetrization operators, we apply \Cref{thm:generalization_error} to obtain a Wasserstein-1 error bound for SGMs with data augmentation. We then develop a theory of equivariant vector fields for SGMs, establishing their key structural properties and demonstrating that equivariant parameterizations are provably superior to standard data augmentation schemes. Finally, we connect these results to equivariant flow matching, revealing a unified perspective on symmetry exploitation in flow-based generative models.

\subsection{Group actions and symmetrization operators}
Let $\CM_b(\R^d)$ be the space of bounded measurable functions on $\R^d$. A \textit{group} is a set $G$ equipped with a group product satisfying the axioms of associativity, identity, and invertibility. Given a group $G$ and a set $\R^d$, a map $\theta:G \times \R^d\to\R^d$ is called a \textit{group action on $\R^d$} if $\theta_g\coloneqq \theta(g, \cdot): \R^d\to\R^d$ is an automorphism on $\R^d$ for all $g\in G$, and $\theta_{g_2}\circ \theta_{g_1} = \theta_{g_2\cdot g_1}$, $\forall g_1, g_2\in G$. By convention, we will abbreviate $\theta(g, x)$ as $gx$ throughout the paper. 

A function $\gamma:\R^d\to \R$ is called \textit{$G$-invariant} if $\gamma\circ \theta_g = \gamma, \forall g\in G$. On the other hand, a probability measure $P\in \CP(\R^d)$ is called \textit{$G$-invariant} if $P = (\theta_g)_* P, \forall g\in G$, where $(\theta_g)_* P\coloneqq P\circ (\theta_g)^{-1}$ is the push-forward measure of $P$ under $\theta_g$. We denote the set of all $G$-invariant probability measures on $\R^d$ as $\CP_G(\R^d)\coloneqq \{P\in\CP(\R^d): P ~\text{is}~ G\text{-invariant}\}$. We make the following assumption on $G$.
\begin{assumption}\label{assumption:group}
$G$ is a group such that the mapping $g: \R^d\to\R^d$ can be written as $g(x)\mapsto A_g x$ for some orthogonal matrix $A_g\in\R^{d\times d}$ for any $g\in G,x\in\R^d$. That is, each $g\in G$ induces a linear isometry.
\end{assumption}
For example, $G$ can include rotations and reflections. Given a group action induced by a group $G$, we introduce two symmetrization operators from \cite{birrell2022structure}, that are useful for our theoretical analysis. 

\textbf{Symmetrization of functions:} $S_G:\CM_b(\R^d)\to \CM_b(\R^d)$,
\begin{equation}
    S_G[\gamma](x)\coloneq \int_G \gamma(gx)\mu_G(\diff{g}) = \E_{\mu_G}[\gamma\circ g(x)],
\end{equation}
where $\gamma\in\CM_b(\R^d)$ and $\mu_G$ is the unique Haar probability measure of $G$. Under \Cref{assumption:group}, $S_G$ is also well-defined for any continuous functions that may be unbounded on $\R^d$, since $G$ is compact.

\textbf{Symmetrization of probability measures (dual operator of $S_G$):} $S^G:\CP(\R^d)\to \CP(\R^d)$, defined for $\gamma\in\CM_b(\R^d)$ by
\begin{equation}
    \E_{S^G[P]}\gamma\coloneq \int_{\R^d} S_G[\gamma]\diff{P(x)}=\E_PS_G[\gamma].
\end{equation}
The operator $S^G$ is commonly understood as \textit{data augmentation} when $P$ is an empirical measure of finite training samples. It has been shown in \cite{birrell2022structure} that both $S_G$ and $S^G$ are projections. We also abuse the notation that if $P$ evolves with time, then $S^G[P]$ means the symmetrization of $P$ at each time.
\begin{definition}
    Given a group action induced by a group $G$ that satisfies \Cref{assumption:group}, we say a vector field $\mathbf{s}:\R^d\times[0,T]\to\R^d$ is $G$-equivariant if 
\begin{equation}
    \mathbf{s}(gx,t) = A_g\cdot \mathbf{s}(x,t)
\end{equation}
for any $x\in\R^d$ and $g\in G$.
\end{definition}
It can be easily verified that if $\rho\in\CP_G(\R^d)$, then its score $\nabla \log\rho$ is $G$-equivariant. In addition to $S_G$ and $S^G$, we propose the following symmetrization operator for vector fields.

\textbf{Symmetrization of vector fields:} 
$S_G^E: (\R^d\times[0,T]\to\R^d)\to (\R^d\times[0,T]\to\R^d)$,
\begin{equation}
    S_G^E[\mathbf{s}](x,t)\coloneq \int_G A_g^\top\cdot \mathbf{s}(gx,t)\mu_G(\diff{g})
\end{equation}
for any vector field $\mathbf{s}$, which is an extension of formula (12) in \cite{lu2024diffusion} for finite groups. It can be shown that $S_G^E[\mathbf{s}]$ is $G$-equivariant for any vector field $\mathbf{s}$. The proof can be found in \cref{sec:proofs}. By the definition of equivariance, we immediately have $S^E_G[\mathbf{s}]=\mathbf{s}$ if $\mathbf{s}$ is $G$-equivariant.

\subsection{SGMs with data augmentation}\label{sec:equivariant_bound}
We apply \Cref{thm:generalization_error} to derive a generalization bound with improved sample complexity in $\mathcal{W}_1$ for learning a $G$-invariant target distribution $\pi$, i.e., $\pi\in\CP_G(\R^d)$.

In contrast to score matching with $N$ i.i.d. samples from $\pi$, we now consider the symmetrized (augmented) empirical measure $\pi^N_G \coloneq S^G[\pi^N]$. For example, if $G$ is a finite group, then $\pi^N_G$ is an empirical measure on $\abs{G}\cdot N$ points that are not independent but conditionally independent samples. We denote by $\pi^{N,\epsilon}_G \coloneq \int P_\epsilon(x,y)\diff{\pi^N_G}(y)$ the early stopping version of the augmented empirical measure, and let $\eta^{N,\epsilon}_G:\R^d\times [0,T]\to[0,\infty)$ be the solution to
\begin{equation}
    \begin{cases}
        \partial_t \rho-\Delta  \rho - \Div(x\cdot \rho)=0 \,\, \text{in}\,\,\R^d\times(0,T],\\
        \rho(0) = \pi^{N,\epsilon}_G \,\,\text{in}\,\,\R^d.
    \end{cases}
\end{equation}
The $G$-invariance of the target distribution $\pi$ provides a significant improvement in the sample efficiency, as shown in the following corollary, when we perform score matching with respect to augmented samples $\eta^{N,\epsilon}_G$.
\begin{corollary}\label{cor:equivariantDSM}
    Assume that the target data measure $\pi$ is $G$-invariant and satisfies \Cref{assumption:tail} and that the noising process follows \eqref{SDE:OU}. Let $A>0$, and suppose the generated distribution $m(T)$ satisfies \eqref{eq:KBE2}. Moreover, assume $\sup_{x\in\R^d,t\in[0,T]}\norm{\nabla_x \mathbf{s}_\theta(x,t)}_2\leq A$ and let $e_{nn} = \CJ_D(\eta^{N,\epsilon}_G,\mathbf{s}_\theta)$ for any $N$ i.i.d. samples $\pi^N$ drawn from $\pi$. Then we have
    \begin{equation}
        \begin{aligned}
        \E\left[\mathcal{W}_1(\pi,m(T))\right]\leq C_d\sqrt{(1-e^{-2\epsilon})}&+ (1-e^{-\epsilon})\Gamma_1(\pi)\\ &+ e^{(1+2A)T}\left(e^{-(T+\epsilon)}\mathcal{W}_1\left(\pi,\mathcal{N}(0,I_d)\right)+\sqrt{T}\sqrt{e_{nn}'}\right),
        \end{aligned}
    \end{equation}
    where
    \begin{equation}
        \begin{aligned}
        e_{nn}'= \E[e_{nn}]&+C_d T e^{-\epsilon}\E\left[\mathcal{W}_1(\pi^{N}_G,\pi)\right]\norm{\mathbf{s}_\theta}^2_{C^2(\R^d\times[0,T])}\\ 
        &+ \frac{4}{ N}\left([2\pi(e^{2(\epsilon+T)}-1)]^{-\frac{d}{2}} R_{\pi}(e^{2(\epsilon+T)}-1)-1\right),
        \end{aligned}
    \end{equation}
    the expectation is taken over $N$ i.i.d. samples drawn from $\pi$, and the $C_d$'s appear above can be different dimensional constants.
\end{corollary}

\begin{remark}
$\E[\mathcal{W}_1(\pi^N_G,\pi)]$ has faster convergence than $\E[\mathcal{W}_1(\pi^N,\pi)]$. First, it is shown in \cite{chen2023sample} that on bounded domains of $\R^d$, we have roughly that
\begin{equation}
    \E[\mathcal{W}_1(\pi^N_G,\pi)]\lesssim
    \begin{cases}
        \left(\frac{1}{\abs{G}N}\right)^{1/d} & \text{if } d\geq3,\\
        \left(\frac{1}{\abs{G} N}\right)^{1/2}\log N & \text{if } d=2,\\
        \frac{\text{diam}(\R^d/G)}{ N^{1/2}}& \text{if } d=1,
    \end{cases}
\end{equation}
if $G$ is finite. Later, \cite{tahmasebisample} extends it to closed Riemannian manifolds with possibly infinite $G$ such that $\E[\mathcal{W}_1(\pi^N_G,\pi)]\lesssim \left(\frac{\text{vol}(\R^d/G)}{N}\right)^{1/d^*}$, where $\text{vol}(\R^d/G)$ is the volume of the quotient manifold $\R^d/G$ and $d^*=\text{dim}(\R^d/G)\geq 3$. More recently, \cite{chen2025robust} extends the convergence estimates to the entire domain $\R^d$ with sub-Weibull assumptions. Such sample complexity gain cannot be derived for the KL or other $f$-divergences without additional regularization.
\end{remark}

\subsection{Properties of score matching with equivariant vector fields} \label{sec:HJBproperty}
\Cref{cor:equivariantDSM} does not explicitly convey the significance of equivariant vector fields in score matching. To that end, we highlight the role of $G$-equivariant vector fields (typically parameterized by neural networks) in score matching, an aspect that has only been addressed experimentally in previous studies. Our rigorous results indicate that it suffices to perform score matching with $G$-equivariant vector fields with respect to an unsymmetrized empirical distribution. This approach will be particularly beneficial when we only have a finite set of \textit{unaugmented} samples, i.e., a non-symmetric empirical distribution drawn from an invariant distribution. Proofs of results in this section can be found in \Cref{sec:proofs}.

First, we show that for any distribution $\rho$, the ISM objective when restricted to $G$-equivariant vector fields, is equivalent to the ISM objective with respect to its symmetrized counterpart.
\begin{theorem}\label{thm:ISM}
    Consider the ISM problem in \eqref{eq:ISM}, in which $\rho$ is not necessarily $G$-invariant. Then for any $G$-equivariant vector field $\mathbf{s}$, we have
     \begin{equation*}
         \CJ_I(\rho,\mathbf{s}) = \CJ_I(S^G[\rho],\mathbf{s}).
     \end{equation*}
\end{theorem}
\begin{remark}
    \cref{thm:ISM} is important for practical implementations, in the sense that the optimal equivariant vector field can be obtained by score matching for raw data without data augmentation. We will demonstrate this point in our numerical simulations in \cref{sec:experiments}.
\end{remark}

Moreover, for the ESM (or equivalently, the DSM) problem of a generic probability measure, the $G$-equivariant minimizer is exactly the score of the symmetrized probability measure, namely:
\begin{proposition}\label{prop:minimumofSM}
     Consider the ESM problem in \eqref{eq:ESM} (or its DSM equivalent), in which $\rho$ is not necessarily $G$-invariant. Denote by $V_G\subset\R^d\times[0,T]\to\R^d$, the subspace of $G$-equivariant vector fields. Then we have 
     \begin{equation*}
         \argmin_{\mathbf{s}\in V_G}\CJ_D(\rho,\mathbf{s})=\argmin_{\mathbf{s}\in V_G}\CJ_E(\rho,\mathbf{s}) = \nabla_x\left[\log\left(S^G[\rho]\right)\right].
     \end{equation*}
\end{proposition}

We propose the following definition as an error quantification for the non-equivariance of a vector field with respect to a $G$-invariant measure $\rho\in\CP_G(\R^d)\times[0,T]$.
\begin{definition}[Deviation from equivariance]
    The deviation from equivariance (DFE) of a vector field $\mathbf{s}$ with respect to $\rho\in\CP_G(\R^d)\times[0,T]$ is defined as
    \begin{equation}
        \mathrm{DFE}(\rho,\mathbf{s})\coloneq \int_0^T\int_{\R^d} \abs{\mathbf{s}-S_G^E[\mathbf{s}]}^2\diff{\rho(s)}\diff{s}.
    \end{equation}
\end{definition}
It is evident that $\mathrm{DFE}(\rho,\mathbf{s})=0$ if $\mathbf{s}$ is $G$-equivariant. Given this definition, we obtain the following decomposition of the ESM and DSM objectives. 
\begin{theorem}\label{prop:equidecomposition}
    For any $\rho\in\CP_G(\R^d)\times[0,T]$ and any vector field $\mathbf{s}$, we have
    \begin{equation}
        \CJ_E(\rho,\mathbf{s}) = \mathrm{DFE}(\rho,\mathbf{s}) + \CJ_E(\rho,S^E_G[\mathbf{s}]).
    \end{equation}
    As DSM and ESM are equivalent objectives, we readily have 
\begin{equation}
    \CJ_D(\rho,\mathbf{s}) = \mathrm{DFE}(\rho,\mathbf{s}) + \CJ_D(\rho,S^E_G[\mathbf{s}])\, , \,\,  \text{for any} \,\, \rho\in\CP_G(\R^d)\times[0,T]\, .
\end{equation}
\end{theorem}
\Cref{prop:equidecomposition} suggests that using $G$-equivariant vector fields can reduce the DSM error for $G$-invariant distributions, at least the DSM error can be reduced if $\mathbf{s}$ is replaced by $S^E_G[\mathbf{s}]$ for score matching.

The following proposition indicates that for any learned distribution $\eta$, its symmetrized counterpart $S^G[\eta]$ is always closer to the $G$-invariant target distribution $\pi$ in the $\mathcal{W}_1$ sense. The $G$-invariance of the generated distribution is guaranteed by the $G$-equivariant vector field $\mathbf{s}_\theta$ and the $G$-invariant initial distribution $\mathcal{N}(0,I_d)$ under \Cref{assumption:group}, as the solution of the backward equation \eqref{eq:KBE2} is unique.

\begin{proposition}\label{prop:symmetrized_prob}
    For any $\eta,\pi\in\CP(\R^d)$, and $\pi$ is $G$-invariant, we have 
    \begin{equation*}
        \mathcal{W}_1(\eta,\pi)\geq\mathcal{W}_1(S^G[\eta],\pi).
    \end{equation*}
\end{proposition}

\subsection{The significance of equivariant vector fields in SGMs}\label{sec:wrap}
With the theoretical results established in \cref{sec:equivariant_bound} and \cref{sec:HJBproperty},
we can now focus on providing quantitative comparisons between using equivariant vector fields and applying data augmentation for learning a $G$-invariant $\pi$.
Our strategy relies on making the generalization bound in \cref{cor:equivariantDSM} as small as possible. In particular, we take a closer look at the terms $e_{nn}$ and $\E[\mathcal{W}_1(\pi^N_G,\pi)]$, which can be improved by selecting an appropriate structure for the vector field or by implementing data augmentation.

The assumption $\CJ_D(\eta_G^{N,\epsilon},\mathbf{s}_\theta)\leq e_{nn}$ in \cref{cor:equivariantDSM} refers to the error of DSM with augmented data, regardless of whether the vector field $\mathbf{s}_\theta$ is $G$-equivariant or not, and we also have $\eta_G^{N,\epsilon} = S^G[\eta^{N,\epsilon}]$.

We denote the model family by $\{\mathbf{s}_\theta, \theta\in\Theta\}$, where $\Theta$ is the set of possible parameters. If we perform data augmentation only, then the best parameter is given by 
\begin{equation}
    \theta_1 = \argmin_{\theta\in\Theta} \CJ_D(\eta_G^{N,\epsilon},\mathbf{s}_\theta).
\end{equation}
Note that under \Cref{assumption:group}, we have $\norm{\nabla_x S_G^E[\mathbf{s}_\theta]}_2\leq\norm{\nabla_x \mathbf{s}_\theta}_2$, so the Lipschitz constraint in \cref{cor:equivariantDSM} does not loose if we replace $\mathbf{s}_\theta$ by its symmetrization $S_G^E[\mathbf{s}_\theta]$. Indeed, we have $\norm{S_G^E[\mathbf{s}_\theta]}_2\leq \norm{\mathbf{s}}_2$, $\norm{\nabla_x S_G^E[\mathbf{s}_\theta]}_2\leq\norm{\nabla_x \mathbf{s}_\theta}_2$, and $\norm{\nabla^2S_G^E[\mathbf{s}_\theta]}_{\text{F}}\leq \norm{\nabla^2\mathbf{s}_\theta}_{\text{F}}$, hence $\norm{S_G^E[\mathbf{s}_\theta]}_{C^2(\R^d\times[0,T])}\leq\norm{\mathbf{s}_\theta}_{C^2(\R^d\times[0,T])}$. If instead we approximate the score of unaugmented data by $\{S_G^E[\mathbf{s}_\theta], \theta\in\Theta\}$, then the best parameter is obtained by solving $\min_{\theta\in\Theta}\CJ_D(\eta^{N,\epsilon},S_G^E[\mathbf{s}_\theta])$. This minimization is equivalent to $\min_{\theta\in\Theta}\CJ_I(\eta^{N,\epsilon},S_G^E[\mathbf{s}_\theta])$ and therefore to $\min_{\theta\in\Theta}\CJ_I(\eta_G^{N,\epsilon},S_G^E[\mathbf{s}_\theta])$ by \Cref{thm:ISM}. 

Note that $\min_{\theta\in\Theta}\CJ_I(\eta_G^{N,\epsilon},S_G^E[\mathbf{s}_\theta])$ is equivalent to $\min_{\theta\in\Theta}\CJ_D(\eta_G^{N,\epsilon},S_G^E[\mathbf{s}_\theta])$. Hence, the best parameter in this case is given by
\begin{equation}
    \theta_2 = \argmin_{\theta\in\Theta} \CJ_D(\eta^{N,\epsilon},S_G^E[\mathbf{s}_\theta]) = \argmin_{\theta\in\Theta} \CJ_D(\eta^{N,\epsilon}_G,S_G^E[\mathbf{s}_\theta]).
\end{equation}
We compare $J_D(\eta_G^{N,\epsilon},\mathbf{s}_{\theta_1})$ with $J_D(\eta^{N,\epsilon}_G,S_G^E[\mathbf{s}_{\theta_2}])$:
\begin{align*}
    \CJ_D(\eta_G^{N,\epsilon},\mathbf{s}_{\theta_1}) &= \text{DFE}(\eta_G^{N,\epsilon},\mathbf{s}_{\theta_1}) + \CJ_D(\eta^{N,\epsilon}_G,S_G^E[\mathbf{s}_{\theta_1}])\\
    &\geq \CJ_D(\eta^{N,\epsilon}_G,S_G^E[\mathbf{s}_{\theta_1}])\\
    &\geq \CJ_D(\eta^{N,\epsilon}_G,S_G^E[\mathbf{s}_{\theta_2}]),
\end{align*}
and so the gap $J_D(\eta_G^{N,\epsilon},\mathbf{s}_{\theta_1})-J_D(\eta^{N,\epsilon}_G,S_G^E[\mathbf{s}_{\theta_2}])$ is at least $\text{DFE}(\eta_G^{N,\epsilon},\mathbf{s}_{\theta_1})$. The analysis above means that if we can always reduce the DSM error $e_{nn}$ using a symmetrized model that is $G$-equivariant while maintaining the statistical efficiency in  $\E[\mathcal{W}_1(\pi^N_G,\pi)]$.

\subsection{A symmetry-aware Wasserstein path-space divergence and equivariant flow matching}\label{sec:equivariantFM}

Since the flow matching objective in \eqref{eq:FM} shares the same $L^2$-regression structure as the DSM or ESM objectives for SGMs, inspired by the Wasserstein path-space divergence and the notion of DFE, we propose the following \textit{symmetry-aware Wasserstein path-space divergence}, in which we assume one of the dynamics always has $G$-invariant marginals. 
\begin{proposition}[Symmetry-aware Wasserstein path-space divergence]\label{def:G_divergence}
    Suppose $\rho_1^{[0,T]}$ and $\rho_2^{[0,T]}$ are the path laws of the following SDEs
    \begin{align*}
        \diff{x_t} &= \mathbf{b}_1(x_t,t)\diff{t} + \boldsymbol{\sigma}(t)\diff{W_t},\quad x_0\sim m_1;\\
        \diff{y_t} &= \mathbf{b}_2(y_t,t)\diff{t} + \boldsymbol{\sigma}(t)\diff{W_t},\quad y_0\sim m_2\in\CP_G(\R^d),
    \end{align*}
    from $t=0$ to $t=T$ respectively. Assume that $\mathbf{b}_2(\cdot,t)$ is always $G$-equivariant for $0\leq t\leq T$ and $\sup_{x\in\R^d, t\in[0,T]}\norm{\nabla \mathbf{b}_1(x,t)}_{2}<\infty$, and we define the symmetry-aware Wasserstein path-space divergence as
    \begin{equation}\label{eq: WPD_G}
        \Xi_G\left(\rho_1^{[0,T]}\| \rho_2^{[0,T]}\right)\coloneqq \mathcal{W}_1(m_1,m_2)+\sqrt{T}\left(\left\|\mathbf{b}_2 - S_G^E[\mathbf{b}_1]\right\|_{L^2(\rho_2)}^2
    + \mathrm{DFE}(\rho_2,\mathbf{b}_1)\right)^{1/2},
    \end{equation}
    where $\rho_2(x,t)$ is the $G$-invariant marginal law of $\rho_2^{[0,T]}$ at time $t$. Then we have $\Xi_G\left(\rho_1^{[0,T]}\| \rho_2^{[0,T]}\right) = \Xi\left(\rho_1^{[0,T]}\| \rho_2^{[0,T]}\right)$.
\end{proposition}
The result is a direct consequence of \Cref{def:WUQ} and \Cref{prop:equidecomposition}. In \Cref{def:G_divergence}, one may reduce the value of the divergence, given $m_1$ and $\mathbf{b}_1$, through replacing $m_1$ by $S^G[m_1]$ (due to \Cref{prop:symmetrized_prob}), and replacing $\mathbf{b}_1$ by $S_G^E[\mathbf{b}_1]$. The DFE may also be viewed as a model-dependent regularization term and can therefore be weighted by an additional hyperparameter. More importantly, the decomposition reveals a fundamental limitation of data- and training-based improvements for learning $G$-invariant path laws: while increasing the sample size or training longer can reduce the first two error terms, neither can decrease the DFE. Similarly, data augmentation may improve the sample-complexity contribution, but it leaves the DFE unchanged. In contrast, an equivariant parametrization eliminates the DFE. 

When $\pi$ is $G$-invariant, \cite{klein2024equivariant} proposes an equivariant OT flow matching where the cost function of the OT is $\tilde{c}(x_0,x_1) = \min_{g\in G}\norm{x_0-\theta_g(x_1)}_2^2$, thereby incorporating the underlying group structure into the coupling. They further showed that the induced flow is $G$-equivariant, and thus setting $\tilde{\rho}(0)=\pi^N$ in \eqref{eq:transport_empirical} is equivalent to setting $\tilde{\rho}(0)=\pi^N_G$; that is, the equivariant OT coupling makes the flow starting from $\pi^N$
equivalent to a flow starting from $\pi^N_G$. The following result shows that equivariant approximation of the velocity field can reduce the error term $e_{nn}$ without increasing the Lipschitz constant $A$ in \Cref{thm:generalization_error_FM}.
\begin{corollary}\label{thm:equivariant_FM}
Suppose $\pi\in\CP_G(\R^d)$. Assume $\mathbf{s}_\theta$ satisfies $\sup_{x\in\R^d,t\in[0,1]}\norm{\nabla_x \mathbf{s}_\theta(x,t)}_2\leq A$ for some $A>0$ and $\tilde{\mathbf{v}}$ is $G$-equivariant for any $N$ i.i.d. samples $\pi^N$ drawn from $\pi$. Let $e_{nn} = \CJ_{FM}(\tilde{\mathbf{v}},\mathbf{s}_\theta,\tilde{\rho})$, then we have
\begin{equation}\label{eq:Gflowmatching}
\E\left[\mathcal{W}_1(\pi,m(1))\right]\leq  \E\left[\mathcal{W}_1(\pi,\pi^N_G)\right] + e^{A}\sqrt{\E[e_{nn}-\mathrm{DFE}(\tilde{\rho}, \mathbf{s}_\theta)]},
\end{equation}
when we generate samples using $S_G^E[\mathbf{s}_\theta]$, where the expectation is taken over $N$ i.i.d. samples drawn from $\pi$. 
\end{corollary}
\begin{proof}
    Since $\tilde{\rho}(1)$, the standard Gaussian, is $G$-invariant and $\tilde{\mathbf{v}}$ is $G$-equivariant, the time-reversal distribution $\tilde{\rho}$ is $G$-invariant for $0\leq t<1$. Noting that the symmetrization of $\mathbf{s}_\theta$ does not increase the Lipschitz constant $A$, the bound follows from \Cref{thm:generalization_error_FM} and \Cref{prop:equidecomposition}.
\end{proof}
\begin{remark}
    The second term on the right-hand side of \eqref{eq:Gflowmatching} is exactly the symmetry-aware Wasserstein path-space divergence $\Xi_G$  for the time-reverse deterministic dynamics: one picks $m_1=m_2=\CN(0,I_d)$, $\mathbf{b}_1 = -\mathbf{s}_\theta$, $\mathbf{b}_2 = -\tilde{\mathbf{v}}$, $T=1$ and    $\boldsymbol{\sigma}(t)\equiv 0$ on the right-hand side of \eqref{eq: WPD_G}.
\end{remark}
\section{Comparison of model equivariance and data augmentation}\label{sec:experiments}
We validate the basic results of our theory in \Cref{sec:equivariant_SGM} on Gaussian mixtures that possess group symmetries. The primary purpose is to emphasize that minimizing the score matching objective with respect to a non-symmetric sample of a $G$-invariant distribution $\pi$ within a class of $G$-equivariant vector fields is better than just augmenting the data through group actions, as is indicated by our analysis described in \Cref{sec:wrap}, and captured by \Cref{thm:ISM} and the DFE term in the symmetry-aware Wasserstein path-space divergence \eqref{eq: WPD_G}.

We consider a mixture of 8 isotropic Gaussians in $\R^{10}$ centered at $[\pm 5, \pm 5,\pm5]$ in the first three coordinates with the remaining seven coordinates centered at zero. Each mode has identity covariance. The distribution is invariant under the group $G = (\mathbb{Z}_2)^3$ of sign flips in the first three coordinates with $|G| = 8$. We report the $\mathcal{W}_1$ distance between the generated distribution and the target distribution. We consider four experimental setups: the first case (\textbf{Equivariant, not augmented}) is where the score network is parametrized to be $G$-equivariant by parametrizing it as 
\begin{align}\label{eq:eqscore}
    \mathbf{s}^G_\theta(x,t) = \frac{1}{|G|}\sum_{g \in G} A_g^\top\mathbf{s}_\theta(A_gx,t),
\end{align}
where $|G|=8$ is the order of the group. The score is trained on $N_{train}$ samples that are not augmented. The second case (\textbf{Equivariant, augmented}) is where the score network is parametrized as in \eqref{eq:eqscore}, and is trained on data that is augmented by applying each group action on each training sample (hence effectively $8\times N_{train}$ samples). The third case (\textbf{Non-equivariant, augmented}) is where the network $\mathbf{s}_\theta$ is trained directly but on augmented training data. The fourth case (\textbf{Non-equivariant, not augmented}) is where the network $\mathbf{s}_\theta$ is trained directly and the training data is not augmented. For each case, the function $\mathbf{s}_\theta$ is parametrized via a fully-connected neural network with 4 hidden layers and 128 nodes per layer. Although our theory assumes a Lipschitz bound on the network, we test the case where no Lipschitz constraint is imposed. The networks are trained over 30000 iterations of the Adam optimizer with learning rate $10^{-3}$ and the batch size is $N_{batch} = 32$. For $N_{train} = 10$, we sample with replacement during the optimization iterations. 

The Wasserstein-1 distance is computed via discrete optimal transport (the earth mover's distance (EMD)) between the generated and target samples using the \emph{Python Optimal Transport} library \cite{flamary2021pot,flamary2024pot}. For each trained score network, we draw $5000$ samples via the reverse-time SDE and compare them against $5000$ samples drawn from the target distribution. For each value of $N_{train}$, we perform $25$ independent runs of each method with a newly sampled training set. The EMD means and standard deviations of the 25 runs are reported in Table~\ref{table:Gaussian10D} and in Figure~\ref{fig:main_figure}.  Notice that the equivariant cases (blue and orange curves) consistently outperform the non-equivariant cases (red and green curves), which corroborates our theoretical analysis for the DFE error. The gap between the \textbf{Non-equivariant, not augmented} and the \textbf{Non-equivariant, augmented} cases (red and green curves) for limited samples (e.g., $N\leq 100$) demonstrates the improvement in the statistical efficiency of  $\E[\W_1(\pi^N_G,\pi)]$ achieved through data augmentation. In contrast, the indistinguishability between the \textbf{Equivariant, not augmented} and the \textbf{Equivariant, augmented} cases (blue and orange curves) confirms \Cref{thm:ISM}, showing that equivariant parameterization automatically obtains the benefits of data augmentation.

\begin{table}[h!]
\caption{$\mathcal{W}_1$ values for an 8-component Gaussian mixture in 10 dimensions. The distribution is symmetric about the axes in the first three coordinates, yielding eight modes at the corners $(\pm5,\pm5,\pm5)$. The remaining seven coordinates are isotropic Gaussians.  }
\label{table:Gaussian10D}
\begin{center}
\begin{tabular}{@{}lcccc@{}}
\multirow{2}{*}{$N_{train}$} & \bf Equivariant,  & \bf Equivariant, & \bf Non-equivariant, & \bf Non-equivariant, \\
& \bf augmented & \bf not augmented & \bf augmented & \bf not augmented \\
\hline \\
$10$    & $\mathbf{3.01 \pm 0.21}$ & $\mathbf{2.94 \pm 0.14}$ & $4.01 \pm 0.28$ & $5.61 \pm 0.63$ \\
$100$   & $\mathbf{2.47 \pm 0.11}$ & $\mathbf{2.43 \pm 0.08}$ & $2.93 \pm 0.12$ & $3.40 \pm 0.29$ \\
$1000$  & $\mathbf{2.49 \pm 0.08}$ & $\mathbf{2.49 \pm 0.10}$ & $2.74 \pm 0.14$ & $2.75 \pm 0.10$ \\
$5000$  & $\mathbf{2.50 \pm 0.06}$ & $\mathbf{2.50 \pm 0.08}$ & $2.72 \pm 0.09$ & $2.78 \pm 0.12$ \\
\end{tabular}
\end{center}
\label{table:experiments10D}
\end{table}

\begin{figure}[h!]
\centering
    \includegraphics[width = 0.67\textwidth]{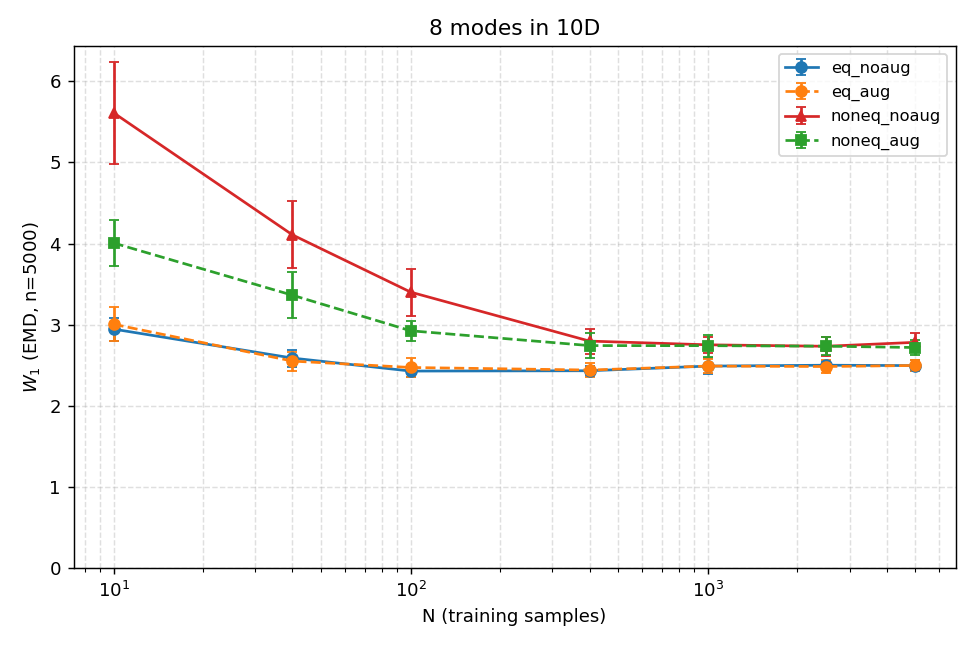}
    \caption{Wasserstein-1 distance as a function of training sample size. The Wasserstein distance is calculated using the earth mover's distance with $5000$ fixed true samples for each distance computation. }
    \label{fig:main_figure}
\end{figure}

\section{Proofs of theoretical results in \Cref{sec:HJBproperty}}\label{sec:proofs}
\paragraph{$G$-equivariance of $S_G^E[\mathbf{s}]$.}
For any $\bar{g}\in G$, we have
    \begin{align*}
        S_G^E[\mathbf{s}](\bar{g}x,t)
        &=\int_G A_g^\top\cdot \mathbf{s}(g\bar{g}x,t)\mu_G(\diff{g})\\
        &=\int_G A_{\bar{g}} A_{\bar{g}}^\top A_g^\top\cdot \mathbf{s}(g\bar{g}x,t)\mu_G(\diff{g})\\
        &= \int_G A_{\bar{g}}  A_{g\circ \bar{g}}^\top\cdot \mathbf{s}(g\bar{g}x,t)\mu_G(\diff{g})\\
        &= A_{\bar{g}}S_G^E[\mathbf{s}](x,t).
    \end{align*}

\begin{proof}[Proof of \cref{thm:ISM}]
    It is sufficient to look at the integration of $x$ over $\R^d$. We have
    \begin{align*}
        \int_{\R^d} \left(\abs{\mathbf{s}}^2+2\nabla\cdot\mathbf{s}\right)S^G[\rho](x)\diff{x}
        &= \int_{\R^d} S_G\left[\abs{\mathbf{s}}^2+2\nabla\cdot\mathbf{s}\right]\rho(x)\diff{x}\\
        &= \int_{\R^d} \abs{\mathbf{s}}^2\rho(x)\diff{x}+ 2\int_{\R^d} S_G\left[\nabla\cdot\mathbf{s}\right]\rho(x)\diff{x},
    \end{align*}
    where the last equality is due to that the module $\abs{\mathbf{s}}$ is $G$-invariant since $\mathbf{s}$ is $G$-equivariant.
    For the second integral, we have
    \begin{align*}
        \int_{\R^d} S_G\left[\nabla\cdot\mathbf{s}\right]\rho(x)\diff{x}
        &= \int_{\R^d} \int_G \sum_{i=1}^d\frac{\partial(\mathbf{s}_i)}{\partial x_i}(gx)\diff{\mu_G(g)}\rho(x)\diff{x}\\
        &= \int_G \int_{\R^d} \sum_{i=1}^d\frac{\partial(\mathbf{s}_i)}{\partial x_i}(gx)\rho(x)\diff{(x)}\diff{\mu_G(g)}\\
        &= \int_G \int_{\R^d} \sum_{i=1}^d\frac{\partial(\mathbf{s}_i)}{\partial x_i}(x)\rho(g^{-1}x)\diff{(g^{-1}x)}\diff{\mu_G(g)}\\
        &= -\int_G \int_{\R^d} \mathbf{s}(x)^\top(A_g\nabla\rho|_{g^{-1}x})\diff{(g^{-1}x)}\diff{\mu_G(g)}\quad\text{(use integration by parts)}\\
        &= -\int_G \int_{\R^d} (A_g^\top\mathbf{s}(x))^\top(\nabla\rho|_{g^{-1}x})\diff{(g^{-1}x)}\diff{\mu_G(g)}\\
        &= -\int_G \int_{\R^d} (\mathbf{s}(g^{-1}x))^\top(\nabla\rho|_{g^{-1}x})\diff{(g^{-1}x)}\diff{\mu_G(g)}\quad\text{(by the equivariance of $\mathbf{s}$)}\\
        &= -\int_G \int_{\R^d} (\mathbf{s}(x))^\top(\nabla\rho(x))\diff{x}\diff{\mu_G(g)}\\
        &= \int_G\int_{\R^d} (\nabla\cdot\mathbf{s})(x)\rho(x)\diff{x}\diff{\mu_G(g)}\\
        &= \int_{\R^d} (\nabla\cdot\mathbf{s})(x)\rho(x)\diff{x}.
    \end{align*}
    Therefore, we have
    \begin{align*}
        \int_{\R^d} \left(\abs{\mathbf{s}}^2+2\nabla\cdot\mathbf{s}\right)S^G[\rho](x)\diff{x}
        = \int_{\R^d} \left(\abs{\mathbf{s}}^2+2\nabla\cdot\mathbf{s}\right)\rho(x)\diff{x}.
    \end{align*}
\end{proof}

To prove \cref{prop:minimumofSM}, we need the following lemma.
\begin{lemma}\label{lemma:scoreofsymmetrization}
    For a generic $\rho\in\CP(\R^d)$, which may not be $G$-invariant, the score formula of its symmetrized measure $S^G[\rho]$, is given by
    \begin{equation*}
        \nabla_x\left[\log\left(S^G[\rho]\right)\right](x) = \frac{\int_G A_g^\top\cdot(\nabla_x\rho)|_{gx}\diff{\mu_G (g)}}{\int_G\rho(gx)\diff{\mu_G (g)}},
    \end{equation*}
    where $(\nabla_x\rho)|_{gx}$ is the gradient of $\rho$ evaluated at $gx$.
\end{lemma}

\begin{proof}[Proof of \cref{lemma:scoreofsymmetrization}]
        \begin{align*}
            \nabla_x\left[\log\left(S^G[\rho]\right)\right](x) &= \nabla_x\left[\log\left(\int_G\rho(gx)\diff{\mu_G (g)}\right)\right]\\
            &= \frac{\nabla_x \int_G\rho(gx)\diff{\mu_G (g)}}{\int_G\rho(gx)\diff{\mu_G (g)}}\\
            &= \frac{ \int_G\nabla_x\rho(gx)\diff{\mu_G (g)}}{\int_G\rho(gx)\diff{\mu_G (g)}}\\
            &= \frac{\int_G A_g^\top\cdot(\nabla_x\rho)|_{gx}\diff{\mu_G (g)}}{\int_G\rho(gx)\diff{\mu_G (g)}}.
        \end{align*}
    \end{proof}

\begin{proof}[Proof of \cref{prop:minimumofSM}]
It suffices to prove the result for the ESM objective for each time $t$, so we omit the time parameter. Let $\R^d/G$ be the quotient space of $\R^d$ by $G$. By the definition in \cref{eq:ESM}, denoting by $\nabla\log\rho|_{gx}$ the score $\nabla\log\rho$ evaluated at $gx$, up to a multiplicative constant $C_G$ the depends on $G$ ($C_G=1$ if $\text{dim}(\R^d/G)<d$ and $C_G=\abs{G}$ if $G$ is finite), we have
    \begin{align*}
        \CJ_E(\rho,\mathbf{s}) &= C_G\int_{\R^d/G}\int_G \abs{\mathbf{s}(gx)-\nabla\log\rho|_{gx}}^2\rho(gx)\diff{\mu_G (g)}\diff{x}\\
        &= C_G\int_{\R^d/G}\int_G \abs{A_g\cdot\mathbf{s}(x)-\nabla\log\rho|_{gx}}^2\rho(gx)\diff{\mu_G (g)}\diff{x}\\
        &= C_G\int_{\R^d/G}\int_G \abs{\mathbf{s}(x)-A_g^\top\cdot\nabla\log\rho|_{gx}}^2\rho(gx)\diff{\mu_G (g)}\diff{x},
    \end{align*}
    where the last equality is due to the group actions in $G$ are isometries. For each $x\in\R^d/G$, regardless of $C_G$, we have
    \begin{align*}
        \nabla_{\mathbf{s}}\left[\int_G \abs{\mathbf{s}(x)-A_g^\top\cdot\nabla\log\rho|_{gx}}^2\rho(gx)\diff{\mu_G (g)} \right] &= 2\int_G \mathbf{s}(x)-A_g^\top\cdot(\nabla\log\rho|_{gx})\rho(gx)\diff{\mu_G (g)}.
    \end{align*}
    Then the stationary point of the above equation is given by
    \begin{equation*}
        \mathbf{s}^*(x) = \frac{\int_G A_g^\top\cdot(\nabla\log\rho|_{gx})\rho(gx)\diff{\mu_G (g)}}{\int_G \rho(gx)\diff{\mu_G (g)}}.
    \end{equation*}
    Note that $\nabla\log\rho|_{gx} = \frac{(\nabla_x \rho)|_{gx}}{\rho(gx)}$. This combined with \cref{lemma:scoreofsymmetrization} proves the claim.
\end{proof}

\begin{proof}[Proof of \cref{prop:equidecomposition}]
    It suffices to prove the equality for each time $t$, thus we will omit the time parameter. Expanding the square, it is equivalent to show that 
    \begin{equation*}
        \int_{\R^d}(\mathbf{s}^\top\nabla\log\rho)\rho(x)\diff{x} = \int_{\R^d} \left(\mathbf{s}^\top S^E_G[\mathbf{s}]-\abs{S^E_G[\mathbf{s}]}^2+S^E_G[\mathbf{s}]^\top \nabla\log\rho\right)\rho(x)\diff{x}.
    \end{equation*}
    First, we show that $\int \mathbf{s}^\top S^E_G[\mathbf{s}]\rho(x)\diff{x}=\int \abs{S^E_G[\mathbf{s}]}^2\rho(x)\diff{x}$. We have
    \begin{align*}
        \text{LHS} = \int_{\R^d} \int_G \mathbf{s}(x)^\top\cdot A_g^\top \mathbf{s}(gx)\diff{\mu_G (g)}\rho(x)\diff{x}
    \end{align*}
    by the definition of the operator $S^E_G$; while 
    \begin{align*}
        \text{RHS}
        &= \int_{\R^d}\int_G\int_G \mathbf{s}(g_1x)^\top A_{g_1}A_{g_2}^\top\mathbf{s}(g_2x)\diff{\mu_G (g_1)}\diff{\mu_G (g_2)}\rho(x)\diff{x}\\
        &= \int_{\R^d}\int_G\int_G \mathbf{s}(g_1x)^\top A_{g_2\circ g_1^{-1}}^\top\mathbf{s}(g_2x)\diff{\mu_G (g_1)}\diff{\mu_G (g_2)}\rho(x)\diff{x}\\
        &=
        \int_G\int_G\int_{\R^d} \mathbf{s}(g_1x)^\top A_{g_2\circ g_1^{-1}}^\top\mathbf{s}(g_2x)\rho(x)\diff{x}\diff{\mu_G (g_1)}\diff{\mu_G (g_2)}\\
        &=
        \int_G\int_G\int_{\R^d} \mathbf{s}(x)^\top A_{g_2\circ g_1^{-1}}^\top\mathbf{s}(g_2\circ g_1^{-1}x)\rho(x)\diff{x}\diff{\mu_G (g_1)}\diff{\mu_G (g_2)}\\
        &=\int_G\int_G\int_{\R^d} \mathbf{s}(x)^\top A_{g}^\top\mathbf{s}(gx)\rho(x)\diff{x}\diff{\mu_G (g)}\diff{\mu_G (g_2)}\\
        &=\int_G\int_{\R^d} \mathbf{s}(x)^\top A_{g}^\top\mathbf{s}(gx)\rho(x)\diff{x}\diff{\mu_G (g)}=\text{LHS}
    \end{align*}
    where the fourth equality is due to the $G$-invariance of $\rho$ and $A_g$ is unitary for any $g\in G$, and the fifth equality is due to that $G$ is unimodular so the Haar measure $\diff_{\mu_{G}}$ is left-, right- and inverse-invariant.

    Then it remains to show that $\int(\mathbf{s}^\top\nabla\log\rho)\rho(x)\diff{x} = \int (S^E_G[\mathbf{s}]^\top\nabla\log\rho)\rho(x)\diff{x}$. Indeed, we have
    \begin{align*}
        \int_{\R^d} (S^E_G[\mathbf{s}]^\top\nabla\log\rho)\rho(x)\diff{x} &=\int_{\R^d}\int_G(A_g^\top\mathbf{s}(gx))^\top \diff{\mu_G (g)}(\nabla\log\rho(x))\rho(x)\diff{x}\\
        &= \int_G\int_{\R^d} \mathbf{s}(gx)^\top A_g (\nabla\log\rho(x))\rho(x)\diff{x}\diff{\mu_G (g)}\\
        &= \int_G\int_{\R^d} \mathbf{s}(gx)^\top (\nabla\log\rho|_{gx})\rho(x)\diff{x}\diff{\mu_G (g)}\\
        &= \int_G\int_{\R^d} \mathbf{s}(x)^\top (\nabla\log\rho(x))\rho(x)\diff{x}\diff{\mu_G (g)}\\
        &=\int_{\R^d} \mathbf{s}(x)^\top (\nabla\log\rho(x))\rho(x)\diff{x},
    \end{align*}
    where the third equality is due to the fact that $\nabla\log\rho$ is $G$-equivariant, and the fourth equality is by a change of variable and $\rho$ is $G$-invariant.
\end{proof}

\begin{proof}[Proof of \cref{prop:symmetrized_prob}]
    Let $\Gamma = \text{Lip}_1(\R^d)$, and $\Gamma_G^{inv}$ be the subspace of $\Gamma$ that consists of $G$-invariant functions. By \cref{assumption:group}, actions in $G$ are 1-Lipschitz. Thus, $S_G[\Gamma]\subset \Gamma$. 
    First note that $S^G[\pi]=\pi$ since $\pi$ is $G$-invariant. Then we have 
    \begin{align*}
        \mathcal{W}_1(S^G[\eta],\pi)&= \mathcal{W}_1(S^G[\eta],S^G[\pi])\\
        &=\sup_{\gamma\in\Gamma}\left\{\E_{S^G[\eta]}[\gamma]-\E_{S^G[\pi]}[\gamma]\right\}\\
        &= \sup_{\gamma\in\Gamma_G^{inv}}\left\{\E_{\eta}[\gamma]-\E_{\pi}[\gamma]\right\}\\
        &\leq \sup_{\gamma\in\Gamma}\left\{\E_{\eta}[\gamma]-\E_{\pi}[\gamma]\right\}=\mathcal{W}_1(\eta,\pi),
    \end{align*}
    where the second equality is by the definition of $\mathcal{W}_1$ metric, and the third equality is due to Theorem 4.6 in \cite{birrell2022structure}.
\end{proof}

\section{Conclusion}\label{sec:conclusion}
We have proposed a new Wasserstein path-space divergence and developed a unified $\W_1$ robustness and generalization theory for flow-based generative models, built on a probabilistic WUP theorem that applies to unbounded domains, accommodates time-varying and possibly degenerate diffusion coefficients, and remains valid in the noiseless setting. Based on the WUP, we obtain $\W_1$ generalization bounds for both score-based diffusion models and flow matching. Specializing the framework to group-invariant targets, we provide the first rigorous error analysis of equivariant SGMs and flow matching, identify a model-form DFE error that quantifies the cost of a non-equivariant parametrization, and show analytically that equivariant inductive bias strictly dominates data augmentation.
Several directions remain open. First, some real-world data distributions, such as image distributions on square domains, may not exhibit exact invariance under group actions. It would therefore be worthwhile to extend our analysis to distributions possessing only approximate symmetries, for instance, when the target distribution is close to a $G$-invariant distribution \cite{birrell2024nonlinear} under a metric such as $\W_1$. On the numerical side, our analysis does not account for the time discretization of SGMs or flow matching; incorporating discretization error, or designing symmetry-preserving numerical integrators within our theoretical framework, would close an important gap between theory and practice. Beyond the OU noising process considered here, extending the framework to nonlinear forward processes such as those in \cite{birrell2024nonlinear,singhal2024s} is a natural next step. Finally, the symmetrization operator under group symmetry can equivalently be viewed as a conditional expectation on the $\sigma$-algebra generated by group orbits; see some preliminary discussion in \cite{birrell2022structure}.
This perspective suggests an extension of our symmetry framework to more general coarse-graining generative models in which only a smaller $\sigma$-algebra is observable, a direction we find particularly promising for scientific applications where partial observability is the norm. In a different direction, while our primary motivation comes from generative modeling, the Wasserstein path-space divergence and the associated uncertainty propagation theorem are formulated at the level of dynamical systems and are not tied to any particular learning architecture. This opens the possibility of applying the framework to model calibration, sensitivity analysis, uncertainty quantification, and operator learning for stochastic and deterministic systems, as previously performed in the KL framework \cite{dupuis2016path,katsoulakis2017scalable,zou2026goal}. We expect that extending the path-space UQ paradigm beyond relative entropy to Wasserstein metrics and the symmetry-aware variants will prove fruitful in a broad range of applications where singular measures, deterministic dynamics, or geometric structure play a central role.

\section{Acknowledgement}
Z. Chen, M. Katsoulakis, B. Zhang are partially funded by AFOSR grant FA9550-21-1-0354. M.K. is partially funded by NSF DMS-2307115.

\bibliographystyle{abbrv}
\bibliography{mybibfile}

\appendix

\section{Empirical score decomposition}\label{appendix:score_decomposition}
We decompose the score function of $\eta^{N,\epsilon}(t)$, appeared in \Cref{thm:generalization_error}, into a monotonic part (gradient of a strongly concave potential) and a residual.
\begin{proposition}\label{lemma:score_decomposition}
Let $z_1, z_2,\dots, z_N$ be $N$ i.i.d. samples on $\R^d$ drawn from $\pi$, then we have
    \begin{equation}
        \nabla_x\log\eta^{N,\epsilon}(x,t) = \frac{-x}{1-e^{-2(\epsilon+t)}}+r^{N,\epsilon}(x,t),
    \end{equation}
    where $\lim_{t\to\infty}\norm{\nabla_x r^{N,\epsilon}}_2=0$ uniformly for all $x\in\R^d$. Moreover, $\E \int_{\R^d}\norm{r^{N,\epsilon}}^2\eta^{N,\epsilon}\diff{x}\leq \frac{e^{-2(\epsilon+t)}}{(1-e^{-2(\epsilon+t)})^2}\Gamma_2(\pi)$,
    where $\Gamma_2(\pi)$ is the second moment of $\pi$, and therefore, $\E \int_{\R^d}\norm{\nabla_x\log\eta^{N,\epsilon}}^2\eta^{N,\epsilon}\diff{x}\leq \frac{4e^{-2(\epsilon+t)}}{(1-e^{-2(\epsilon+t)})^2}\Gamma_2(\pi) + \frac{2d}{1-e^{-2(\epsilon+t)}}$.
\end{proposition}
\begin{proof}
    First, we have
    \begin{align*}
        \eta^{N,\epsilon}(x,t) = \frac{C_{\epsilon,t}}{N}\sum_{i=1}^N e^{-\frac{\norm{x-e^{-(\epsilon+t)}z_i}^2}{2(1-e^{-2(\epsilon+t)})}},
    \end{align*}
    where $C_{\epsilon,t} = [2\pi(1-e^{-2(\epsilon+t)})]^{-\frac{d}{2}}$. Then, we have 
    \begin{align*}
        \nabla_x\log\eta^{N,\epsilon}(x,t) &= \nabla_x\log\left(\sum_{i=1}^N e^{-\frac{\norm{x-e^{-(\epsilon+t)}z_i}^2}{2(1-e^{-2(\epsilon+t)})}}\right)\\
        &= \frac{-1}{1-e^{-2(\epsilon+t)}}\frac{\sum_{i=1}^N(x-e^{-(\epsilon+t)}z_i) e^{-\frac{\norm{x-e^{-(\epsilon+t)}z_i}^2}{2(1-e^{-2(\epsilon+t)})}}}{\sum_{i=1}^N e^{-\frac{\norm{x-e^{-(\epsilon+t)}z_i}^2}{2(1-e^{-2(\epsilon+t)})}}}\\
        &= \frac{-x}{1-e^{-2(\epsilon+t)}} + \frac{1}{1-e^{-2(\epsilon+t)}}\frac{\sum_{i=1}^N(e^{-(\epsilon+t)}z_i) e^{-\frac{\norm{x-e^{-(\epsilon+t)}z_i}^2}{2(1-e^{-2(\epsilon+t)})}}}{\sum_{i=1}^N e^{-\frac{\norm{x-e^{-(\epsilon+t)}z_i}^2}{2(1-e^{-2(\epsilon+t)})}}}\\
        &\coloneqq \frac{-x}{\sigma_t^2} + \frac{\sum_{i=1}^Nw_i(x)a_i}{\sigma_t^2},
    \end{align*}
    where 
    \[\sigma_t^2 = 1-e^{-2(\epsilon+t)},\quad w_i(x) = \frac{e^{-\frac{\norm{x-e^{-(\epsilon+t)}z_i}^2}{2(1-e^{-2(\epsilon+t)})}}}{\sum_{i=1}^N e^{-\frac{\norm{x-e^{-(\epsilon+t)}z_i}^2}{2(1-e^{-2(\epsilon+t)})}}},\quad a_i = e^{-(\epsilon+t)}z_i.\] 
    By a direct calculation and using the fact that $\sum_{i=1}^N w_i(x) \equiv 1$ for any $x\in\R^d$, we have the Jacobian:
    \begin{equation*}
        \nabla_x(\nabla_x\log\eta^{N,\epsilon}) = \frac{-I}{\sigma_t^2} + \frac{1}{2\sigma_t^4}\sum_{i,j=1}^N w_i(x)w_j(x)(a_i-a_j)(a_i-a_j)^\top,
    \end{equation*}
    where 
    \begin{align*}
        \norm{\frac{1}{2\sigma_t^4}\sum_{i,j=1}^N w_i(x)w_j(x)(a_i-a_j)(a_i-a_j)^\top}_2\leq \frac{\max_{i,j}\norm{a_i-a_j}^2}{2\sigma_t^4} = \frac{e^{-2(\epsilon+t)}\max_{i,j}\norm{z_i-z_j}^2}{2\sigma_t^4},
    \end{align*}
    which approaches 0 when $t\to\infty$.

    For the second part, note that by Jensen's inequality, we have
    \begin{align*}
        \norm{r^{N,\epsilon}}^2 = \frac{\norm{\sum_{i=1}^Nw_i(x)a_i}^2}{\sigma_t^4}\leq \frac{\sum_{i=1}^N w_i(x)\norm{a_i}^2}{\sigma_t^4}.
    \end{align*}
    Recalling the formula for $w_i$, we have $w_i(x)\eta^{N,\epsilon} = \frac{C_{\epsilon,t}}{N}e^{-\frac{\norm{x-e^{-(\epsilon+t)}z_i}^2}{2(1-e^{-2(\epsilon+t)})}}$, so that
    \begin{align*}
        \int_{\R^d}\norm{r^{N,\epsilon}}^2\eta^{N,\epsilon}\diff{x}&\leq \int_{\R^d}\frac{\sum_{i=1}^N w_i(x)\norm{a_i}^2}{\sigma_t^4}\eta^{N,\epsilon}\diff{x}\\
        &=\frac{C_{\epsilon,t}}{N\sigma_t^4}\sum_{i=1}^N\int_{\R^d}e^{-\frac{\norm{x-e^{-(\epsilon+t)}z_i}^2}{2(1-e^{-2(\epsilon+t)})}}e^{-2(\epsilon+t)}\norm{z_i}^2\diff{x}\\
        &= \frac{e^{-2(\epsilon+t)}}{N\sigma_t^4}\sum_{i=1}^N\norm{z_i}^2 = \frac{e^{-2(\epsilon+t)}}{(1-e^{-2(\epsilon+t)})^2}\cdot\frac{1}{N}\sum_{i=1}^N\norm{z_i}^2.
    \end{align*}
    Taking expectations for $z_1,\dots,z_N\sim\pi$, we have
    \begin{equation}
        \E \int_{\R^d}\norm{r^{N,\epsilon}}^2\eta^{N,\epsilon}\diff{x} \leq \frac{e^{-2(\epsilon+t)}}{(1-e^{-2(\epsilon+t)})^2}\Gamma_2(\pi).
    \end{equation}
    Therefore, we have
    \begin{align*}
        &\E\int_{\R^d}\norm{\nabla_x\log\eta^{N,\epsilon}}^2\eta^{N,\epsilon}\diff{x}\\
        &\leq \E\int_{\R^d} \frac{2\norm{x}^2}{(1-e^{-2(\epsilon+t)})^2}\eta^{N,\epsilon}\diff{x} + \E\int_{\R^d}2\norm{r^{N,\epsilon}}^2\eta^{N,\epsilon}\diff{x}\\
        &\leq \E\int_{\R^d}\frac{2\norm{x}^2}{(1-e^{-2(\epsilon+t)})^2}\frac{C_{\epsilon,t}}{N}\sum_{i=1}^N e^{-\frac{\norm{x-e^{-(\epsilon+t)}z_i}^2}{2(1-e^{-2(\epsilon+t)})}}\diff{x} + \frac{2e^{-2(\epsilon+t)}}{(1-e^{-2(\epsilon+t)})^2}\Gamma_2(\pi)\\
        &= \frac{2}{(1-e^{-2(\epsilon+t)})^2}
        \E\left(e^{-2(\epsilon+t)}\frac{1}{N}\sum_{i=1}^N\norm{z_i}^2+d(1-e^{-2(\epsilon+t)})\right) + \frac{2e^{-2(\epsilon+t)}}{(1-e^{-2(\epsilon+t)})^2}\Gamma_2(\pi)\\
        &= \frac{4e^{-2(\epsilon+t)}}{(1-e^{-2(\epsilon+t)})^2}\Gamma_2(\pi) + \frac{2d}{1-e^{-2(\epsilon+t)}}.
    \end{align*}
\end{proof}

\section{Lemmas for \Cref{thm:generalization_error}}\label{appendix:WUP}
The following lemma provides a bound for the ISM error. A version of the result was established in \cite[Proposition B.2]{mimikos2024score} for bounded domains; we give the proof here using our \Cref{lemma:exponential_decay} for  $\R^d$.
\begin{lemma}\label{lemma:DSM_1}
    Let $\pi_i$ ($i=1,2$) be two probability measures in $\R^d$ and $\rho_i$ the corresponding solutions to \eqref{eq:OUevolution} with $\rho_i(0) = \pi_i$. There exists a dimensional constant $C_d>0$ such that 
    \begin{equation}
        \abs{\CJ_I(\rho_2,\mathbf{s}_\theta)-\CJ_I(\rho_1,\mathbf{s}_\theta)}\leq C_d T \sup_{t\in[0,T]}\mathcal{W}_1(\rho_1(t),\rho_2(t))\norm{\mathbf{s}_\theta}^2_{C^2(\R^d\times[0,T])}\leq C_d T \mathcal{W}_1(\pi_1,\pi_2)\norm{\mathbf{s}_\theta}^2_{C^2(\R^d\times[0,T])}.
    \end{equation}
\end{lemma}
\begin{proof}
    Using \eqref{eq:ISM} and letting $g_\theta = \norm{\mathbf{s}_\theta}_2^2+2\Div(\mathbf{s}_\theta)$, we have 
    \begin{align*}
        \abs{\CJ_I(\rho_2,\mathbf{s}_\theta)-\CJ_I(\rho_1,\mathbf{s}_\theta)} 
        &= \abs{\int_0^T\int_{\R^d}\left(\norm{\mathbf{s}_\theta}_2^2+2\Div(\mathbf{s}_\theta)\right)\diff{(\rho_2(s)-\rho_1(s))}\diff{s}}\\
        &= \abs{\int_0^T\int_{\R^d}g_\theta\diff{(\rho_2(s)-\rho_1(s))}\diff{s}}\\
        &\leq \sup_{0\leq t\leq T}\norm{\nabla g_\theta}_2 \int_0^T \mathcal{W}_1(\rho_2(s),\rho_1(s))\diff{s}\\
        &\leq T \sup_{0\leq t\leq T}\norm{\nabla g_\theta}_2 \sup_{0\leq t\leq T}\mathcal{W}_1(\rho_2(t),\rho_1(t)).
    \end{align*}
    Since there exists some $C_d>0$, such that $\sup_{0\leq t\leq T}\norm{\nabla g_\theta}_2\leq C_d \norm{\mathbf{s}_\theta}^2_{C^2(\R^d\times[0,T])}$, the lemma is proved using \Cref{lemma:exponential_decay} for $\sup_{0\leq t\leq T}\mathcal{W}_1(\rho_2(t),\rho_1(t))$.
\end{proof}

\begin{lemma}\label{lemma:DSM_2}
Let $m_0\in\CP(\R^d)$ be a probability density that has a finite first moment and $m_0\log(m_0)\in L^1(\R^d)$ and $\rho:\R^d\times[0,T]\to\R$ solves the PDE in \eqref{eq:OUevolution} with $\rho(0) = m_0$. Then we have  
    \begin{equation}\label{eq:score_difference}
        4\norm{\nabla\sqrt{\rho}}_2^2 = \int_{\R^d} m_0\log(m_0)-\rho(T)\log(\rho(T))\diff{x} + dT.
    \end{equation}    
\end{lemma}
\begin{proof}
    Multiplying the equation $\partial_t \rho-\Delta \rho -\Div( x\cdot\rho)=0$ with $f'(\rho)$ for a smooth function $f\in C^\infty\left((0,\infty)\right)$ and integrating over space and time, we obtain
    \begin{equation}\label{eq:score_difference_1}
        \int_{\R^d}\int_{0}^T f'(\rho)\partial_t \rho \diff{t}\diff{x} = \int_{0}^T\int_{\R^d}f'(\rho)\Delta\rho \diff{x}\diff{t} + \int_{0}^T\int_{\R^d}f'(\rho)\Div(x\cdot\rho)\diff{x}\diff{t}.
    \end{equation}
    For the left-hand side of \eqref{eq:score_difference_1}, we have 
    \begin{equation*}
        \int_{\R^d}\int_{0}^T f'(\rho)\partial_t \rho \diff{t}\diff{x} = \int_{\R^d} f(\rho(T))-f(m_0)\diff{x}.
    \end{equation*}
    For the first term on the right-hand side of \eqref{eq:score_difference_1}, using integration by parts, we have
    \begin{align*}
        \int_{0}^T\int_{\R^d}f'(\rho)\Delta\rho \diff{x}\diff{t} &= \int_{0}^T\int_{R^{d-1}} \sum_{i=1}^d f'(\rho)\frac{\partial \rho}{\partial x_i}\Big|_{x_i=-\infty}^\infty \diff{(x\backslash\{x_i\})}\diff{t} - \int_{0}^T\int_{\R^d}f''(\rho)\norm{\nabla\rho}^2 \diff{x}\diff{t},
    \end{align*}
    where we use $x\backslash\{x_i\}$ to denote the integration over $\R^d$ except for the $i$-th coordinate $x_i$. For each $i\in[d]$ and $t\in(0,T]$, $\rho$ is bounded and hence is $f'(\rho)$ since $f$ is smooth, and $\frac{\partial\rho}{\partial x_i}\to0$ as $x_i\to\pm\infty$ by using formula \eqref{eq:transition_probability}: $\rho(x,t)=\int P_{t}(x,y)\diff{m_0(y)}$ and applying the dominated convergence theorem. So we have 
    \begin{equation*}
        \int_{0}^T\int_{\R^d}f'(\rho)\Delta\rho \diff{x}\diff{t} = - \int_{0}^T\int_{\R^d}f''(\rho)\norm{\nabla\rho}^2 \diff{x}\diff{t}.
    \end{equation*}
    For the second term on the right-hand side of \eqref{eq:score_difference_1}, applying an integration by parts again, we have 
    \begin{align*}
        \int_{0}^T\int_{\R^d}f'(\rho)\Div(x\cdot\rho)\diff{x}\diff{t}
        &= \int_{0}^T\int_{R^{d-1}} \sum_{i=1}^d \rho f'(\rho)x_i\Big|_{x_i=-\infty}^\infty \diff{(x\backslash\{x_i\})}\diff{t} - \int_{0}^T\int_{\R^d}\rho f''(\rho)x\cdot\nabla\rho \diff{x}\diff{t}.
    \end{align*}
    For each $i\in[d]$ and $t\in(0,T]$, $x_i\rho\to0$ as $x_i\to\pm\infty$ using the formula $\rho(x,t)=\int P_{t}(x,y)\diff{m_0(y)}$ and applying the dominated convergence theorem combined with the assumption that $m_0$ has a finite first moment. So we have
    \begin{equation*}
        \int_{0}^T\int_{\R^d}f'(\rho)\Div(x\cdot\rho)\diff{x}\diff{t} = -\int_{0}^T\int_{\R^d}\rho f''(\rho)x\cdot\nabla\rho \diff{x}\diff{t}.
    \end{equation*}
    Combining all the results and pick in particular $f(x)=x\log x$ and $f''(x)=x^{-1}$, and we have 
    \begin{align*}
        \int_{\R^d} \rho(T)\log(\rho(T))-m_0\log(m_0)\diff{x} 
        &= - \int_{0}^T\int_{\R^d}\frac{\norm{\nabla\rho}^2}{\rho} \diff{x}\diff{t} -\int_{0}^T\int_{\R^d}x\cdot\nabla\rho \diff{x}\diff{t}\\
        &= - \int_{0}^T\int_{\R^d}\frac{\norm{\nabla\rho}^2}{\rho} \diff{x}\diff{t} + d\int_{0}^T\int_{\R^d}\rho \diff{x}\diff{t}\\
        &= - \int_{0}^T\int_{\R^d}\frac{\norm{\nabla\rho}^2}{\rho} \diff{x}\diff{t} +  dT,
    \end{align*}
    where the second equality follows from an integration by parts. Notice that $4\norm{\nabla\sqrt{\rho}}_2^2 = \int_{0}^T\int_{\R^d}\frac{\norm{\nabla\rho}^2}{\rho} \diff{x}\diff{t}$ and we finish the proof.
\end{proof}

\section{WUP on the torus $R\mathbb{T}^d$}\label{appendix:torus}
In this section of the appendix, we provide a proof for the Wasserstein Uncertainty Propagation Theorem when the domain is a torus $\Omega=R\mathbb{T}^d$, using the reflection coupling introduced in \cite{eberle2016reflection}. The resulting Lipschitz bound of test functions has exponential decay in time, and is sharper than Eq. (10) in \cite[Theorem 3.1]{mimikos2024score}. As a result, we will be able to derive a sharper $\mathcal{W}_1$ generalization bound than \cite[Theorem 3.3]{mimikos2024score}. For simplicity, we assume that $\boldsymbol{\sigma}(t)\equiv\sqrt{2}$ as in \cite{mimikos2024score}.

\begin{theorem}[Wasserstein Uncertainty Propagation on $R\mathbb{T}^d$]\label{thm:WUP_torus}
    Let $\Omega = R\mathbb{T}^d$ and  $\mathbf{b}_1,\mathbf{b}_2:\Omega\times[0,T]\to\Omega$ be differentiable vector fields and $m_1,m_2\in\CP(\Omega)$. Probability measures $\rho_i$ $(i=1,2)$ are given by
\begin{equation}\label{eq:WUP_1_torus}
    \partial_t \rho_i-\Delta\rho_i-\Div(\rho_i \mathbf{b}_i)=0,\,\, \rho_i(0)=m_i.
\end{equation}
Let
\begin{equation}\label{wq:WUP_error_torus}
    \epsilon\coloneq\norm{\mathbf{b}_2-\mathbf{b}_1}_{L^2(\rho_2)}\coloneq \left(\int_0^T\int_{\Omega} \abs{(\mathbf{b}_2-\mathbf{b}_1)(x,t)}^2 \rho_2(x,t)\diff{x}\diff{t}\right)^\frac{1}{2}.
\end{equation}
Then if $c\coloneq\sup_{x\in\Omega, t\in[0,T]}\norm{\nabla \mathbf{b}_1(x,t)}_{2}<\infty$, we have 
    \begin{equation}\label{eq:WUP1_torus}
    \mathcal{W}_1(\rho_2(T),\rho_1(T))\leq C_2e^{-C_1T}(\mathcal{W}_1(m_2,m_1)+\sqrt{T}\epsilon),
    \end{equation}
    where $C_1 = \frac{2}{D^2}e^{-cD^2/8}$, $C_2 = 2e^{cD^2/8}$, and $D=\pi R\sqrt{d}$.
\end{theorem}
The proof is identical to that of the full $\R^d$ case in \Cref{thm:WUP}, except that we can use the following lemma to obtain an exponential-decay-in-time Lipschitz bound for the test functions when the domain is bounded.
\begin{lemma}\label{lemma:gradient_bound_0_torus}
Suppose $\phi$ is the solution of the backward equation
\begin{equation}\label{eq:KBE_general_0}
    \partial_t\phi+\Delta\phi+\mathbf{b}\cdot\nabla\phi=0\,\, \text{in}\,\,\Omega\times[0,T),\quad\phi(x,T)=\psi(x)\,\, \text{in}\,\,\Omega,
    \end{equation} 
    where $a\in\R$ is a constant, $\norm{\nabla \mathbf{b}(x,t)}_{2}\leq c$ uniformly in $x\in\Omega$ and $t\in[0,T)$, and $\psi$ is $L$-Lipschitz. Then $\phi(x,t)$ is $C_2e^{-C_1(T-t)}L$-Lipschitz for any $t\in[0,T)$, where $C_1 = \frac{2}{D^2}e^{-cD^2/8}$, $C_2 = 2e^{cD^2/8}$, and $D=\pi R\sqrt{d}$ is the diameter of $\Omega$.
\end{lemma}
\begin{proof}
    Again, using the Feynman-Kac formula, the solution of \eqref{eq:KBE_general_0} can be expressed as
    \begin{equation}
        \phi(x,t) = \E\left[\psi(x_T)\mid x_t=x\right],
    \end{equation}
    where $\diff{x_t} = \mathbf{b}(x_t,t)\diff{t}+ \sqrt{2}\diff{W_t}$. Let $y_t$ satisfies the same SDE as $x_t$, then we have
    \begin{equation}\label{eq:F-K_torus_0}
        \abs{\phi(x,t)-\phi(y,t)} \leq\E\left[\abs{\psi(x_T)-\psi(y_T)}\mid x_t=x,y_t=y\right],
    \end{equation}
    where the conditional expectation is taken jointly on $x_T$ and $y_T$. Since $\psi$ is $L$-Lipschitz, $\abs{\psi(x_T)-\psi(y_T)}\leq L\norm{x_T-y_T}_2$. Set $z_t = x_t - y_t$, $r_t = \norm{z_t}$ and for $r_t>0$, $e_t = z_t/r_t$. Define a second Brownian motion via reflection across the hyperplane orthogonal to $e_t$:
    \begin{equation}
        \diff{\widetilde{W}_t} = (I-2e_te_t^\top)\diff{W}_t\quad\text{for}\,\, t<\tau,
    \end{equation}
    where $\tau = \inf\{t:r_t=0\}$ is the coupling time; after $\tau$, we set $\widetilde{W}_t = W_t$ and $x_t=y_t$. By L\'evy's characterization, $\widetilde{W}_t$ is a standard Brownian motion. Define $\diff{y_t} = \mathbf{b}(y_t,t)\diff{t}+ \sqrt{2}\diff{\widetilde{W}_t}$, then $x_t$ and $y_t$ still have the law of solutions to the same SDE. For $t<\tau$, 
    \begin{equation}
        \diff{z_t} = [\mathbf{b}(x_t,t)-\mathbf{b}(y_t,t)]\diff{t} + \sqrt{2}(\diff{W}_t - \diff{\widetilde{W}_t}) = [\mathbf{b}(x_t,t)-\mathbf{b}(y_t,t)]\diff{t} + 2\sqrt{2}e_te_t^\top\diff{W}_t.
    \end{equation}
    The noise has rank one, aligned with $e_t$, with covariance $\inner{z_t^{(i)}}{z_t^{(j)}} = 8 e^{(i)}_te^{(j)}_t\diff{t}$. Apply It\^o's formula for $r_t=\norm{z_t}$, using $\partial_i\norm{z_t} = z_t^{(i)}/\norm{z_t}$ and $\partial_{ij}\norm{z_t}=(\delta_{ij}-z_t^{(i)}z_t^{(j)}/\norm{z_t}^2)/\norm{z_t}$:
    \begin{equation}
        \diff{r_t} = \inner{e_t}{\diff{z}_t} + \frac{1}{2}\sum_{ij}\frac{\delta_{ij}-e^{(i)}_te^{(j)}_t}{r_t}\cdot 8 e^{(i)}_te^{(j)}_t\diff{t}.
    \end{equation}
    The second term on the right-hand side vanishes, since
    \begin{equation}
        \sum_{ij}(\delta_{ij}-e^{(i)}_te^{(j)}_t)e^{(i)}_te^{(j)}_t = \norm{e_t}^2-\norm{e_t}^4 = 0,
    \end{equation}
    as $\norm{e_t}=1$. For the drift term $\inner{e_t}{\diff{z}_t}$, note that
    \begin{equation}
        \abs{\inner{e_t}{\mathbf{b}(x_t,t)-\mathbf{b}(y_t,t)}}\leq\norm{\mathbf{b}(x_t,t)-\mathbf{b}(y_t,t)}\leq c\norm{z_t}=cr_t,
    \end{equation}
    and $\inner{e_t}{2\sqrt{2}e_te_t^\top\diff{W}_t}=2\sqrt{2}e_t^\top\diff{W}_t = 2\sqrt{2}\diff{B}_t$, where $B_t$ is a 1D standard Brownian motion. Hence, 
    \begin{equation}\label{eq:SDE_rs}
        \diff{r}_t = \beta_t\diff{t} + 2\sqrt{2}\diff{B}_t, \quad \abs{\beta_t}\leq cr_t.
    \end{equation}
    Next, we construct $f\in C^2([0,D])$ that is increasing, concave, and satisfies $\mathcal{L}_rf\leq-C_1f$, where
    \begin{equation}
        \mathcal{L}_rf\coloneqq crf'(r) + 4f''(r)
    \end{equation}
    is the worst-case generator of $r_t$.

    Define 
    \begin{equation}
        \varphi(r)=e^{-cr^2/8}, \quad\Phi(r) = \int_0^r \varphi(u)\diff{u}, \quad r\in[0,D].
    \end{equation}
    Since $\varphi$ is decreasing with $\varphi(0)=1$ and $\varphi(D)= e^{-cD^2/8}$, we have
    \begin{equation}
        \varphi(D)r\leq\Phi(r)\leq r.
    \end{equation}
    Set 
    \begin{equation}
        C_1 = \frac{2}{D^2}e^{-cD^2/8}=\frac{2\varphi(D)}{D^2},
    \end{equation}
    and define
    \begin{equation}
        g(r) = 1-\frac{C_1}{4}\int_0^r\frac{\Phi(u)}{\varphi(u)}\diff{u}.
    \end{equation}
    We verify that $g(D)\geq 3/4$ as follows. Using $\Phi(u)\leq u$ and $1/\varphi(u)\leq 1/\varphi(D)$ on $[0,D]$:
    \begin{equation}
        \int_0^D\frac{\Phi(u)}{\varphi(u)}\diff{u}\leq\frac{1}{\varphi(D)}\int_0^D u\diff{u}=\frac{D^2}{2\varphi(D)}.
    \end{equation}
    Therefore, 
    \begin{equation}
        g(D) \geq 1-\frac{C_1}{4}\frac{D^2}{2\varphi(D)} = 1- \frac{1}{4}\cdot\frac{2\varphi(D)}{D^2}\cdot\frac{D^2}{2\varphi(D)}=\frac{3}{4}.
    \end{equation}
    Hence, $g$ is decreasing on $[0,D]$ with $g(0)=1$ and $g(D)\geq\frac{3}{4}$.

    We define $f$ as
    \begin{equation}
        f(r) = \int_0^r\varphi(u)g(u)\diff{u}, \quad r\in[0,D].
    \end{equation}
    Then $f(0)=0$, $f$ is increasing and concave, since $\varphi g$ is decreasing and positive. We then show that $\mathcal{L}_rf\leq-C_1f$ for such $f$.

    Note that
    \begin{equation}
        f'(r) = \varphi(r)g(r),\quad f''(r) = \varphi'(r)g(r) + \varphi(r)g'(r).
    \end{equation}
    Since $\varphi'(r) = -\frac{cr}{4}\varphi(r)$, we have
    \begin{equation}
        f''(r) = -\frac{cr}{4}\varphi(r)g(r) + \varphi(r)g'(r).
    \end{equation}
    Thus,
    \begin{equation}
        \mathcal{L}_rf = cr\varphi(r)g(r) + 4\left[-\frac{cr}{4}\varphi(r)g(r) + \varphi(r)g'(r)\right] = 4\varphi(r)g'(r).
    \end{equation}
    From the definition of $g$, $g'(r) = -\frac{C_1}{4}\Phi(r)/\varphi(r)$, so
    \begin{equation}
        \mathcal{L}_rf(r) = -C_1\Phi(r).
    \end{equation}
    Since $g(u)\leq 1$,
    \begin{equation}
        f(r) = \int_0^r\varphi(u)g(u)\diff{u}\leq \int_0^r\varphi(u)\diff{u} = \Phi(r).
    \end{equation}
    Combining, we have $\mathcal{L}_rf\leq-C_1f$.

    Next, we prove a useful inequality for $f$. Since $f'(r) = \varphi(r)g(r)$ and both factors are decreasing on $[0,D]$,
    \begin{equation}
        \varphi(D)g(D)\leq f'(r)\leq \varphi(0)g(0) = 1.
    \end{equation}
    Integrating from $0$ to $r$ and note that $f(0) = 0$, we have
    \begin{equation}\label{eq:inequality_for_f}
        \varphi(D)g(D)r\leq f(r) \leq r.
    \end{equation}
    Using $g(D)\geq 3/4$ and $\varphi(D)= e^{-cD^2/8}$, we have
    \begin{equation}
        \frac{3}{4}e^{-cD^2/8}r\leq f(r)\leq r,
    \end{equation}
    hence,
    \begin{equation}\label{eq:biLipschitz}
        r\leq \frac{4}{3}e^{cD^2/8}f(r)\leq 2e^{cD^2/8}f(r).
    \end{equation}

    Apply It\^o's formula to $f(r_t)$ using the SDE for $r_t$ in \eqref{eq:SDE_rs}:
    \begin{equation}
        \diff{f(r_t)} = [\beta_tf'(r_t) + 4f''(r_t)]\diff{r_t} + 2\sqrt{2}f'(r_t)\diff{B_t}.
    \end{equation}
    Since $f'\geq 0$, $f''\geq 0$ (concavity), and $\abs{\beta_t}\leq cr_t$,
    \begin{equation}
        \beta_tf'(r_t) + 4f''(r_t) \leq cr_tf'(r_t) + 4f''(r_t) = \mathcal{L}_rf(r_t)\leq -C_1f(r_t).
    \end{equation}
    Therefore,
    \begin{equation}
        \diff{f(r_t)}\leq -C_1f(r_t)\diff{t} + 2\sqrt{2}f'(r_t)\diff{B_t}.
    \end{equation}
    Multiplying both sides by $e^{C_1t}$, we have
    \begin{equation}
        \diff{\left(e^{C_1t}f(r_t)\right)} \leq 2\sqrt{2}e^{C_1t}f'(r_t)\diff{B_t}.
    \end{equation}
    The right side is a true martingale ($f'\leq 1$ is bounded). Taking expectations (conditioned on $x_t=x$ and $y_t =y$),
    \begin{equation}
        \E[e^{C_1T}f(r_T) - e^{C_1t}f(r_t)]\leq 0,
    \end{equation}
    so that $\E[f(r_T)]\leq \E[e^{-C_1(T-t)}f(r_t)] = e^{-C_1(T-t)}f(\norm{x-y})$.

    Combining this bound with \eqref{eq:biLipschitz}, we have
    \begin{equation}
        \E r_{T}\leq 2e^{cD^2/8}\E f(r_T)\leq 2e^{cD^2/8}e^{-C_1(T-t)}f(\norm{x-y})\leq 2e^{cD^2/8}e^{-C_1(T-t)}\norm{x-y},
    \end{equation}
    where the last inequality is due to \eqref{eq:inequality_for_f}. Plugging it into \eqref{eq:F-K_torus_0}, we have
    \begin{equation}
        \abs{\phi(x,t)-\phi(y,t)}\leq L\E r_T\leq 2e^{cD^2/8}e^{-C_1(T-t)}\norm{x-y}.
    \end{equation}
\end{proof}
\end{document}